\documentclass{article}
\usepackage{dilab_arxiv}
\usepackage{enumitem}
\usepackage{flafter}
\usepackage{wrapfig}
\usepackage{placeins}
\usepackage{needspace}

\theoremstyle{plain}
\newtheorem{theorem}{Theorem}[section]
\newtheorem{lemma}[theorem]{Lemma}
\newtheorem{proposition}[theorem]{Proposition}
\newtheorem{corollary}[theorem]{Corollary}
\theoremstyle{remark}
\newtheorem{remark}[theorem]{Remark}
\theoremstyle{plain}

\newcommand{\KL}{D_{\mathrm{KL}}}
\renewcommand{\eqref}[1]{equation~\ref{#1}}

\AddToHook{env/theorem/begin}{\samepage}
\AddToHook{env/lemma/begin}{\samepage}
\AddToHook{env/proposition/begin}{\samepage}
\AddToHook{env/corollary/begin}{\samepage}

\title{One Proposal for Every Margin: Zero-Shot Amortized Sequential Importance Sampling for Binary Matrices}
\runningtitle{MarginFlow: One Proposal for Every Margin}
\date{arXiv preprint, September 2026}
\paperlogo{\includegraphics[height=1.5cm]{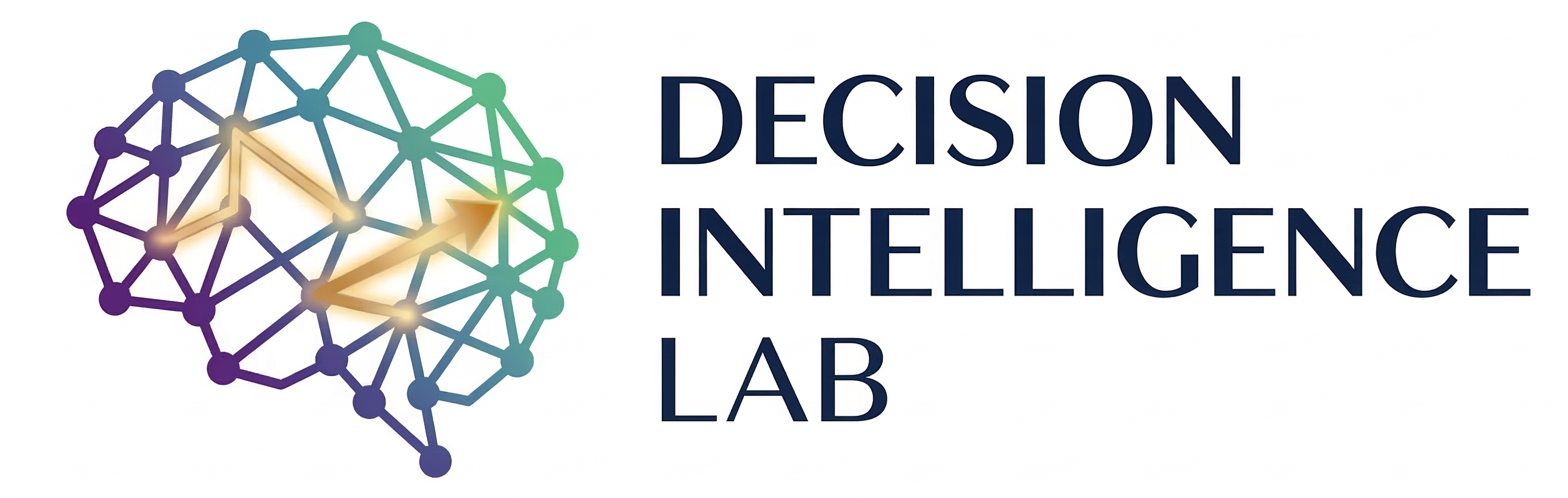}}
\author{
  Ruishuo Chen$^{1}$, Weijia Li$^{2}$, Xun Wang$^{1}$,
  Yu Chen$^{1}$, Leheng Cai$^{2}$, Longbo Huang$^{1\,\text{\faEnvelope}}$
  \\[0.3em]\normalfont
  $^1$Institute for Interdisciplinary Information Sciences, Tsinghua University\\
  $^2$Department of Statistics and Data Science, Tsinghua University\\
  \text{\faEnvelope}\ Correspondence:
  \href{mailto:longbohuang@tsinghua.edu.cn}{\texttt{longbohuang@tsinghua.edu.cn}}
}
\hypersetup{pdftitle={One Proposal for Every Margin: Zero-Shot Amortized Sequential Importance Sampling for Binary Matrices},pdfauthor={Ruishuo Chen, Weijia Li, Xun Wang, Yu Chen, Leheng Cai, Longbo Huang}}

\begin{document}
\maketitle
\thispagestyle{fancy}

\begin{abstract}
In ecology, psychometrics, and the analysis of social and financial networks, binary matrices are often analyzed conditional on their observed row and column sums, which restricts the problem to a finite sample space of matrices with the same margins. Two fundamental problems are to count this space and to sample uniformly from it. Sequential importance sampling (SIS) addresses both with independent weighted samples and an unbiased count estimator, but its efficiency depends critically on the proposal distribution. Existing proposals are analytically designed, and their accuracy can vary substantially with the margins. We show that the ideal SIS proposal, under which every weight equals the count and the variance vanishes, is exactly the policy of a generative flow network (GFlowNet) with unit reward on every matrix that has the given margins. We therefore propose MarginFlow, a framework that turns the design of the proposal into a learning problem and amortizes it across margins by exploiting their self-similarity. Every partial matrix is itself an instance with reduced margins, so one set transformer that reads the remaining margins serves every margin. We train MarginFlow on a pool of 1904 margins and evaluate it zero-shot on 1190 held-out margins, synthetic and real, from $3\times3$ to $870\times6$. On 1187 of the 1190 margins it matches or beats the best of 31 analytically designed configurations, chosen post hoc for each margin, and its median effective sample fraction is 99.8\%. On the 56 margins where that best loses more than one nat of effective sample size, MarginFlow wins every one and raises the median effective sample fraction from 10.3\% to 94.1\%.
\end{abstract}

\section{Introduction}
\label{sec:intro}

Given prescribed row sums \(r\) and column sums \(c\), let
\[
\Omega(r,c)
=
\{X\in\{0,1\}^{m\times n}:
X\mathbf 1=r,\ X^\top\mathbf 1=c\}
\]
denote the finite sample space of binary matrices with these margins.
Two fundamental computational problems are to evaluate its cardinality
\(Z(r,c)=|\Omega(r,c)|\) and to sample uniformly from it. Ecologists compare the co-occurrence patterns observed across a set of sites with random binary matrices that preserve every species' prevalence and every site's richness \citep{connor1979assembly,neal2024pattern}. 
The same comparison is routine wherever a binary table is judged against its own row and column totals, in psychometrics
\citep{chen2005exact,draxler2025testing}, social network analysis
\citep{snijders1991enumeration,neal2025stopping}, cancer genomics
\citep{gobbi2014fast}, and the study of financial and ecological networks
\citep{glasserman2023maximum,sun2025efficient}.

\looseness=-1
The two computational tasks are closely linked through completion counts.
If the first \(i-1\) rows have been fixed, each feasible next row defines a child of the current partial matrix, and exact uniform sampling selects that
child with probability proportional to its number of valid completions. Exact dynamic programming exploits this
recursion to provide both counting and uniform sampling, but its state space grows exponentially with the number of columns, so even moderate sizes are out of reach
\citep{miller2013exact}. Markov-chain methods instead generate matrices with the prescribed
margins and approximate the uniform distribution
\citep{verhelst2008efficient,gotelli2012statistical,strona2014fast,
fosdick2018configuring,fu2026spectral,nie2026snake}.
Their draws are typically correlated, and the chains do not directly
estimate the number of feasible matrices.

\begin{figure}[t]
\centering
\includegraphics[width=\linewidth]{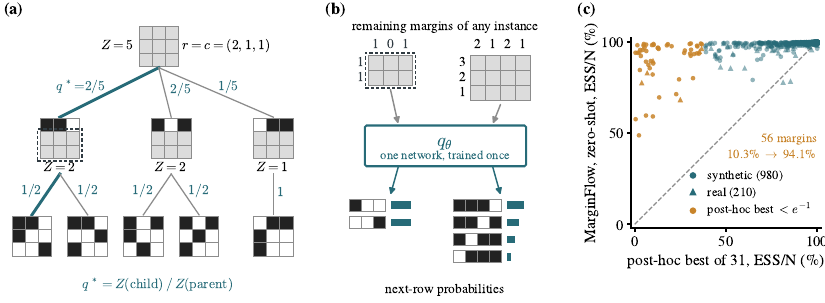}
\caption{(a) The five binary matrices with margins $r=c=(2,1,1)$, built row by row. Each state carries its number of completions $Z$, and the ratio $q^*$ on each edge is the zero-variance proposal: along the bold path $q^*(X)=\frac25\cdot\frac12=\frac15$, so the weight $1/q^*(X)$ is the count $5$. (b) The outlined remainder is itself an instance with margins $r'=(1,1)$, $c'=(1,0,1)$; one set transformer reads the remaining margins of any instance and returns next-row probabilities. (c) Effective sample fraction of MarginFlow, zero-shot, against the post-hoc best of 31 configurations on 1190 held-out margins; circles are synthetic margins, triangles real ones, amber the 56 where that best loses over one nat.}
\label{fig:intro}
\end{figure}

\looseness=-1
Sequential importance sampling (SIS) instead addresses both tasks with
independent weighted samples and an unbiased estimator of the count
\citep{snijders1991enumeration,chen2005sequential}. It builds a matrix row by
row from a proposal over feasible next rows and weights each completed matrix by the reciprocal of its proposal probability. Its efficiency depends critically on the proposal, whose accuracy can vary substantially with the margins \citep{harrison2013importance}. On some margin families the mismatch is large enough that SIS requires
exponentially many draws to avoid severe underestimation
\citep{bezakova2012negative}. A long line of work has therefore refined the analytic proposal
\citep{blanchet2009efficient,blitzstein2011sequential,harrison2013importance,glasserman2023maximum}. These proposals differ in how they approximate the completion counts behind the ideal next-row probabilities. We instead learn those probabilities directly from the proposal's own draws.

In this paper, we view the proposal through the lens of generative flow networks (GFlowNets) \citep{bengio2021flow,dasilva2025right}, amortized samplers that build objects stepwise and end at each in proportion to its reward. We show that, with unit reward on every matrix that has margins $(r,c)$, the flow at any partial matrix equals its number of completions, while the flow at the initial state equals the total count. Consequently, selecting each child in proportion to its flow yields exactly the ideal zero-variance SIS proposal, as Figure~\ref{fig:intro}(a) shows. A GFlowNet learns this policy by enforcing flow consistency on the matrices it samples itself \citep{whitammer2022trajectory,tiapkin2024entropy,fawkes2026ftb}, so the proposal can be learned without the count ever being known.

Training a separate GFlowNet per margin, however, costs far more than any analytically designed proposal. We therefore propose MarginFlow, which amortizes the learning across all margins by exploiting their self-similarity. Once $k$ rows are filled, what remains is the same problem on the remaining rows and reduced column sums, so every partial matrix met while sampling is itself an instance, as Figure~\ref{fig:intro}(b) outlines. A set transformer that reads the remaining margins and scores a candidate row is therefore a proposal for every margin at once. We train one such transformer over the column sums on a pool of 1904 margins from six synthetic families and published ecological, mutualistic-network, and psychometric tables, with every test dataset held out by source.

We then evaluate it zero-shot on 1190 held-out margins, synthetic and real, from $3\times3$ to $870\times6$. The baseline is the best of 31 analytically designed proposal configurations, chosen post hoc for each margin. That choice takes all 31 runs and an effective sample size estimated from their draws, and the estimate misses the rare heavy weights behind the exponential underestimate above, so the baseline is stronger than any analytically designed proposal a user can run.

MarginFlow matches or beats this post-hoc best on 1187 of 1190 margins in Figure~\ref{fig:intro}(c), with a median effective sample fraction of 99.8\%. On the 56 margins where that best loses more than one nat, it wins every one and lifts the median from 10.3\% to 94.1\%. One proposal, learned once and never tuned, thus takes over from the analytically designed proposals and the choice among them.

Our contributions are as follows.
\begin{itemize}[leftmargin=1.2em,itemsep=2pt,topsep=2pt]
\item We establish the equivalence between the zero-variance SIS proposal and the forward policy of a GFlowNet with unit reward on every binary matrix that has the given margins, whose total flow is the number of such matrices. This turns analytic proposal construction into a learning problem.
\item We propose MarginFlow, which amortizes this learning across all margins through the self-similarity of the problem, since every partial matrix is itself an instance. One set transformer that reads the remaining margins is trained on a pool of margins. It then serves zero-shot as the proposal on margins it has never seen, with no per-instance training, tuning, or selection.
\item On 1190 held-out margins MarginFlow matches or beats the post-hoc best of 31 analytically designed configurations on all but three, with a median effective sample fraction of 99.8\%. On the 56 margins where that post-hoc best loses more than one nat, it keeps a median of 94.1\% and wins every one, so the learned proposal holds where analytical design gives out.
\end{itemize}

\section{Preliminaries}
\label{sec:prelim}

\subsection{Sequential importance sampling for fixed margins}
\label{sec:sis}

Recall that \(\Omega(r,c)\) denotes the set of \(m\times n\) binary matrices
with row sums \(r\) and column sums \(c\), and let
\(Z=|\Omega(r,c)|\). We assume throughout that
\(\Omega(r,c)\neq\varnothing\).
 Uniform draws from $\Omega(r,c)$ are the null distribution of a conditional test, and $Z$ is the normalizing constant of a likelihood conditioned on the margins \citep{rasch1960probabilistic,chen2005exact,harrison2013importance}.

Sequential importance sampling (SIS) obtains the count and the draws from one procedure that builds a matrix $X\in\Omega(r,c)$ row by row \citep{snijders1991enumeration,chen2005sequential}. Once rows $x_1,\dots,x_{i-1}$ are placed, the partial matrix $X_{<i}$ leaves an instance with margins $(r_i,\dots,r_m)$ and $c-\sum_{k<i}x_k$, and we write $Z(X_{<i})$ for its number of completions, so that $Z(X_{<1})=Z$. A row $x_i\in\{0,1\}^n$ with $|x_i|=r_i$ is feasible if $Z(X_{\le i})>0$, which a Gale--Ryser test on the reduced margins decides without counting \citep{chen2005sequential}.

The sampling step draws each row from a proposal $q(x_i\mid X_{<i})$ over the feasible rows, so a finished matrix is drawn with probability $q(X)=\prod_{i=1}^m q(x_i\mid X_{<i})$ rather than with the uniform probability $1/Z$. The importance step corrects for this by attaching to the draw the weight $w(X)=1/q(X)$. If $q$ gives positive probability to every feasible row, then $\mathbb{E}_q[w]=\sum_{X\in\Omega(r,c)}q(X)/q(X)=Z$, so the mean weight over $N$ draws is an unbiased estimate of the count, and the draws reweighted by $w$ estimate any expectation under the uniform distribution, the p-value among them.

The estimate is unbiased whatever the proposal, but its precision is set by the variance of the weights, and the usual summary of that variance is the effective sample fraction
\begin{equation}
\frac{\mathrm{ESS}}{N}=\frac{\bigl(\sum_{\ell=1}^N w_\ell\bigr)^2}{N\sum_{\ell=1}^N w_\ell^2}.
\label{eq:ess}
\end{equation}
The effective sample size $\mathrm{ESS}$ approximates the number of equally weighted draws with the same Monte Carlo precision as the $N$ weighted draws. Thus, improving SIS amounts to increasing $\mathrm{ESS}/N$, which is one for constant weights and approaches $1/N$ when a single weight dominates.

One proposal makes all the weights equal. Taking each row with probability proportional to the number of completions it leaves,
\begin{equation}
q^*(x_i\mid X_{<i})=\frac{Z(X_{\le i})}{Z(X_{<i})},
\label{eq:qstar}
\end{equation}
telescopes along any matrix to $q^*(X)=Z(X_{\le m})/Z(X_{<1})=1/Z$, so every weight is exactly $Z$, as Figure~\ref{fig:intro}(a) traces on a $3\times3$ example. Evaluating $q^*$, however, is as hard as the count itself, since every $Z(X_{<i})$ is a count of the same kind \citep{miller2013exact}.

Every classical proposal is therefore a closed-form stand-in for \eqref{eq:qstar} computed from the reduced margins \citep{chen2005sequential,harrison2013importance}. How closely it tracks $q^*$ depends on the margins, and where it falls short the effective sample fraction collapses and the count can be underestimated by an exponential factor \citep{bezakova2012negative}. In this paper, we instead learn the stand-in, a network trained as a GFlowNet policy that returns $q(x_i\mid X_{<i})$ for any reduced margins.

\subsection{Generative flow networks}
\label{sec:gfn}

A GFlowNet is trained to sample objects with probability proportional to a reward \citep{bengio2021flow,bengio2023gflownet}. An object is built from an initial state $s_0$ by a sequence of actions, each moving to a child state in a directed acyclic graph, until a terminal state $x$ is reached, where a reward $R(x)>0$ is given. A forward policy $P_F(s'\mid s)$ assigns probabilities to the children of every state, and the goal is a policy that ends at $x$ with probability $R(x)/Z_R$, where $Z_R=\sum_x R(x)$.

Flows describe such a policy. Assign to every state a flow $F(s)$ and to every edge a flow $F(s\to s')$ such that inflow equals outflow at every state but $s_0$ and the outflow of a terminal state is its reward. Then $F(s_0)=Z_R$, and the policy $P_F(s'\mid s)=F(s\to s')/F(s)$ has the required terminal distribution \citep{bengio2021flow}. When the graph is a tree, as when states record their history, each has one parent, the edge flow into $s'$ is $F(s')$, and $F(s)$ is the total reward of the terminals below it.

\looseness=-1
Training needs no flow values, only the reward of each sampled terminal state, and neither objective below parameterizes $F(s)$ for $s\neq s_0$. Trajectory balance \citep{whitammer2022trajectory} parameterizes the policy $P_F(\cdot\mid s;\theta)$ together with a scalar $Z_\theta$ and asks every complete trajectory $\tau=(s_0\to\dots\to x)$, with $P_F(\tau;\theta)=\prod_t P_F(s_{t+1}\mid s_t;\theta)$, to satisfy $Z_\theta P_F(\tau;\theta)=R(x)$, the form the condition takes on a tree. The log-variance objective, VarGrad \citep{richter2020vargrad}, is the same condition with $\log Z_\theta$ replaced by its optimal value for a batch of $B$ trajectories, the mean of the log-ratio $\log R(x)-\log P_F(\tau;\theta)$, so that the loss is the variance of that log-ratio across the batch,
\begin{equation}
\begin{aligned}
\mathcal{L}_{\mathrm{TB}}(\tau)&=\bigl(\log Z_\theta+\log P_F(\tau;\theta)-\log R(x)\bigr)^2,\\
\mathcal{L}_{\mathrm{LV}}(\tau_{1:B})&=\operatorname{Var}_{b}\bigl[\log R(x_b)-\log P_F(\tau_b;\theta)\bigr].
\end{aligned}
\label{eq:tb}
\end{equation}
Both vanish on every trajectory exactly when $P_F(\tau;\theta)=R(x)/Z_R$, with $\log Z_\theta$, or the mean log-ratio, equal to $\log Z_R$. The second learns no partition function, which matters when one network serves instances with their own $Z_R$ \citep{zhang2023robust}. The next section uses it with reward one.

\section{MarginFlow: one proposal for all margins}
\label{sec:method}

This section builds MarginFlow in two steps. Section~\ref{sec:flow} recasts SIS for fixed margins as a GFlowNet, so that the ideal proposal becomes a policy learned from its own draws, and Section~\ref{sec:network} designs one network that learns it for all margins at once.

\subsection{Counting as a flow}
\label{sec:flow}

In this subsection we examine SIS for fixed margins as a GFlowNet whose reward is one on every matrix with the given margins. We prove that the zero-variance proposal is the flow-proportional policy, that the trajectory-balance residual of any proposal is its log weight up to a constant, and that the log-variance loss of Section~\ref{sec:gfn} is the variance of the log weights. We then show what training does while the count stays unknown and how the loss relates to the effective sample fraction we report. Proofs are in Appendix~\ref{app:proofs}, and Appendix~\ref{app:rows} shows how row errors add up over a matrix.

Fix margins $(r,c)$. The states are the partial matrices $X_{<i}$ for $i=1,\dots,m+1$, with the empty matrix $X_{<1}$ as $s_0$. The actions at $X_{<i}$ are the feasible rows $x_i$, the terminal states are the complete matrices $X\in\Omega(r,c)$, and every terminal state has reward $R(X)=1$. A partial matrix records the rows placed so far, so each state has one parent and the graph is a tree. The Gale--Ryser test keeps every trajectory inside $\Omega(r,c)$, so every trajectory reaches a complete matrix. Throughout, $p$ is the uniform distribution on $\Omega(r,c)$, $p(X)=1/Z$. A policy $q$ over feasible rows draws a matrix with probability $q(X)=\prod_i q(x_i\mid X_{<i})$ and gives it the weight $w(X)=1/q(X)$, as in Section~\ref{sec:sis}.

\begin{lemma}[Flows count completions]
\label{lem:flow}
The flow through a partial matrix is its number of completions, $F(X_{<i})=Z(X_{<i})$, and the total flow $Z_R=F(X_{<1})$ is the count $Z$. The flow-proportional policy $P_F(x_i\mid X_{<i})=Z(X_{\le i})/Z(X_{<i})$ is the zero-variance proposal $q^*$.
\end{lemma}

The lemma is the picture in Figure~\ref{fig:intro}(a). Each node carries its number of completions, and the flow-proportional policy divides that number among the children. It also says why $q^*$ is out of reach, since every flow value is itself a count. The next result extends the correspondence to any proposal, good or bad, and shows that the GFlowNet losses measure what SIS cares about.

\begin{theorem}[SIS is a GFlowNet with reward one]
\label{thm:sis}
Let $q_\theta$ be any policy that gives every feasible row positive probability, with weights $w_\theta(X)=1/q_\theta(X)$.
\begin{enumerate}[label=(\roman*),leftmargin=2em,itemsep=1pt,topsep=2pt]
\item For every $X\in\Omega(r,c)$, the trajectory-balance condition of \eqref{eq:tb} reads $Z_\theta\,q_\theta(X)=1$, and its residual is $\log Z_\theta-\log w_\theta(X)$.
\item For matrices $X_1,\dots,X_B$ drawn from $q_\theta$, the two losses of \eqref{eq:tb} are
\begin{equation}
\begin{aligned}
\mathcal{L}_{\mathrm{TB}}(X_b)&=\bigl(\log Z_\theta-\log w_\theta(X_b)\bigr)^2,\\
\mathcal{L}_{\mathrm{LV}}(X_{1:B})&=\operatorname{Var}_b\bigl[\log w_\theta(X_b)\bigr].
\end{aligned}
\label{eq:lossw}
\end{equation}
\item $\operatorname{Var}_{q_\theta}[\log w_\theta(X)]=0$ if and only if $q_\theta=q^*$, in which case every weight equals $Z$.
\end{enumerate}
\end{theorem}

In words, a GFlowNet trained on this tree is an SIS sampler whose loss is the spread of its own log weights. Analytically designed proposals try to keep this spread small, and only $q^*$ removes it.

The loss of Theorem~\ref{thm:sis} is a statistic of one batch. The next theorem, the on-policy equivalence of trajectory balance and reverse KL \citep{richter2020vargrad,whitammer2023gflownets}, says what it optimizes in expectation. Let $H(q_\theta)$ be the entropy of $q_\theta$ over $\Omega(r,c)$, the distribution of the draws.

\begin{theorem}[Training maximizes the entropy of the draws]
\label{thm:grad}
Let $q_\theta$ be differentiable in $\theta$ with full support, and let the $B$ matrices be drawn independently from $q_\theta$ and held fixed in the differentiation, as in on-policy training. Then
\begin{equation}
\mathbb{E}\bigl[\nabla_\theta\mathcal{L}_{\mathrm{LV}}\bigr]
=2\tfrac{B-1}{B}\,\nabla_\theta\KL\bigl(q_\theta\,\|\,p\bigr)
=-2\tfrac{B-1}{B}\,\nabla_\theta H(q_\theta),
\label{eq:grad}
\end{equation}
and the batch mean of $\log w_\theta$ is an unbiased estimate of $H(q_\theta)=\log Z-\KL(q_\theta\,\|\,p)$.
\end{theorem}

With $p$ uniform the divergence is $\log Z-H(q_\theta)$, so each expected update raises the entropy of the draws, which is largest when they are uniform on $\Omega(r,c)$.

\begin{remark}[Why the count comes from SIS]
\label{rem:count}
A GFlowNet trained with trajectory balance also learns a scalar $\log Z_\theta$, which looks like a free estimate of the count. For a given policy, however, trajectory balance fits $\log Z_\theta$ to the mean log weight, which by Theorem~\ref{thm:grad} is $\log Z-\KL(q_\theta\,\|\,p)$, so the fitted value falls short of $\log Z$ by whatever divergence training has not yet removed. The mean weight of SIS is unbiased under any policy with full support, so we take the proposal from the GFlowNet and the count from SIS. Because $\log Z$ enters \eqref{eq:grad} only as a constant of each instance, one network can also train on a pool of margins, with no $\log Z_\theta$ per margin ever needed.
\end{remark}

The loss is measured on the network's own draws, and we report the effective sample fraction of \eqref{eq:ess}. The following standard facts of importance sampling \citep{kong1994sequential,agapiou2017importance} relate the two through the R\'enyi divergence of order two, $D_2(p\,\|\,q)=\log\sum_{X}p(X)^2/q(X)$.

\begin{proposition}[The loss and the effective sample fraction]
\label{prop:ess}
For a policy $q$ with full support on $\Omega(r,c)$ and $N$ independent draws from it, as $N\to\infty$,
\begin{equation}
\frac{\mathrm{ESS}}{N}\to\frac{Z^2}{\mathbb{E}_q[w^2]}=e^{-D_2(p\,\|\,q)},
\label{eq:essd2}
\end{equation}
and the relative mean squared error of the count estimate from $N$ draws is $(e^{D_2(p\,\|\,q)}-1)/N$. If $q=p(1+h)$, where $h$ is the relative deviation of $q$ from $p$ and $\mathbb{E}_p[h]=0$, then $D_2(p\,\|\,q)$, $\operatorname{Var}_q[\log w]$ and $2\KL(q\,\|\,p)$ all equal $\mathbb{E}_p[h^2]$ to second order in $h$, so they agree near $q^*$.
\end{proposition}

In the limit the effective sample fraction is thus a R\'enyi divergence, $-\log(\mathrm{ESS}/N)=D_2(p\,\|\,q)$ in nats, and the relative error of the count is set by it. Close to the optimum, the training loss of Theorem~\ref{thm:sis}, the divergence of Theorem~\ref{thm:grad} and the effective sample fraction are one quantity, so the GFlowNet is trained on the precision of the SIS count.

\subsection{One network for all margins}
\label{sec:network}

We now design MarginFlow's network around the self-similarity and symmetry of the problem, so that one network trained with the loss of Section~\ref{sec:flow} serves as the proposal for all margins. Figure~\ref{fig:method} draws the design, from the row types at one state to one proposal step and one training step.

\begin{figure}[t]
\centering
\includegraphics[width=0.8\linewidth]{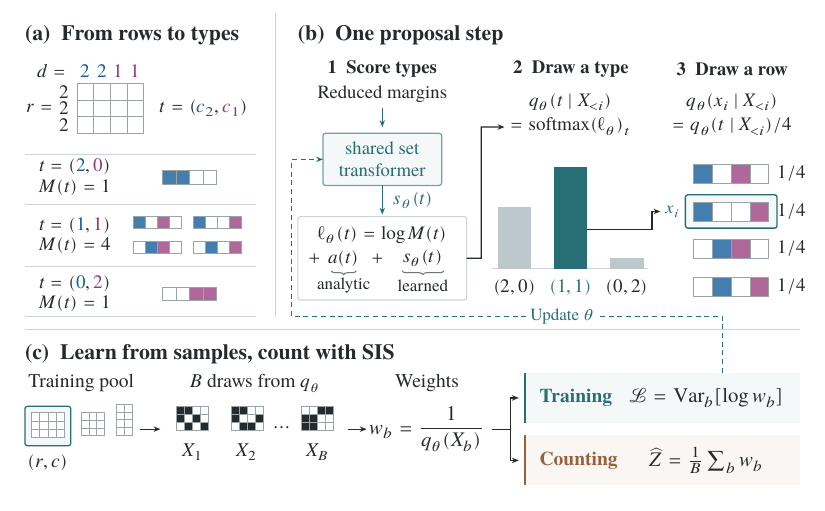}
\caption{(a) At a state with remaining row sums $(2,2,2)$ and reduced column sums $(2,2,1,1)$, the six feasible rows fall into three types by their numbers of ones in the columns of each reduced sum. (b) One step of the proposal, the logit of \eqref{eq:logit} for every feasible type, a softmax over the types, and a uniform draw within the chosen type. (c) Training minimizes the variance of the log weights of $B$ draws with one margin from the pool, and SIS averages the same weights for the count.}
\label{fig:method}
\end{figure}

\begin{proposition}[The optimal policy reads the reduced margins]
\label{prop:reduced}
Let a partial matrix $X_{<i}$ have remaining row sums $r_{\ge i}=(r_i,\dots,r_m)$ and reduced column sums $d=c-\sum_{k<i}x_k$. Then $q^*(x_i\mid X_{<i})$ depends on $X_{<i}$ only through $(r_{\ge i},d)$, and it is unchanged by permuting the rows after row $i$. For every column permutation $\pi$, $q^*(\pi x_i\mid r_{\ge i},\pi d)=q^*(x_i\mid r_{\ge i},d)$. In particular, feasible rows with the same number of ones in the columns of each reduced sum are equally likely.
\end{proposition}

The proposition lets one network that reads reduced margins act at every state of every instance. The symmetry reduces the work further. At a state with reduced column sums $d$, let $h_v$ be the number of columns $j$ with $d_j=v$, and for a row $x_i$ let $c_v(x_i)$ be the number of its ones in those columns. We call $t(x_i)=(c_v(x_i))_v$ the type of the row, and Figure~\ref{fig:method}(a) sorts the rows of one state by type. By the proposition, $q^*(x_i\mid X_{<i})$ depends on $x_i$ only through its type, and the $M(t)=\prod_v\binom{h_v}{c_v}$ rows of a type $t$ are all feasible or all infeasible, which the Gale--Ryser test decides once per type.

We therefore let the network output a distribution $q_\theta(t\mid X_{<i})$ over the feasible types, which shrinks its output from up to $\tbinom{n}{r_i}$ rows to the far fewer types, and we draw a row uniformly among the $M(t)$ rows of the chosen type, the step that Figure~\ref{fig:method}(b) traces. A matrix is then drawn with log-probability
\begin{equation}
\log q_\theta(X)=\sum_{i=1}^m\Bigl[\log q_\theta\bigl(t(x_i)\mid X_{<i}\bigr)-\log M\bigl(t(x_i)\bigr)\Bigr].
\label{eq:type}
\end{equation}

We design the logit of a feasible type as
\begin{equation}
\ell_\theta(t)=\log M(t)+a(t)+s_\theta(t)
\label{eq:logit}
\end{equation}
and $q_\theta(t\mid X_{<i})$ is the softmax of $\ell_\theta$ over the feasible types. The term $\log M(t)$ counts the rows a type holds, a combinatorial quantity computed exactly outside the network. The term $a(t)$ is the log weight that the proposal of \citet{harrison2013importance} gives to each row of type $t$, written out in Appendix~\ref{app:prior} and ablated in Appendix~\ref{app:ablation}, and it gives training a strong prior to explore from. The term $s_\theta(t)$ is what the network learns. The softmax gives every feasible type positive probability, and the uniform draw within a type passes it on to every matrix in $\Omega(r,c)$. The mean importance weight under any fixed trained policy $q_\theta$ thus remains an unbiased estimator of $Z$, learning the proposal changes only the spread of the weights around $Z$, and Theorem~\ref{thm:sis}(iii) applies.

\looseness=-1
The network is a set transformer \citep{lee2019set} that reads the remaining instance as a set of tokens and returns $s_\theta(t)$. Let $R=m-i+1$ be the number of rows still to place and $n'$ the number of columns with $d_j>0$. The reduced column sums enter as one token per distinct value $v$, carrying $v$ and the number $h_v$ of columns with that sum. The remaining row sums enter in the same way, one token per distinct value and the number of rows that have it, and the current row sum $r_i$ has a token of its own. These tokens hold their values as fractions of $R$ and $n'$, so they describe the shape of the remaining instance, and one last token carries $R$ and $n'$ themselves.

Four pre-norm attention layers of width 256, with no positional encoding, turn the tokens into embeddings, so the symmetry of Proposition~\ref{prop:reduced} holds by construction. This pass runs once per state, and every feasible type at that state shares it. A small head then scores each type $t$ from the embeddings of the values $v$ it touches together with its counts $c_v$, so the cost of a step is one encoder pass and one head evaluation per type. We initialize the last layer of the head at zero, so training starts from the analytically designed proposal.

We train on a pool of margins as in Figure~\ref{fig:method}(c). At each step we draw an instance, sample $B$ matrices with its margins from the current policy, and minimize the variance of their log weights, the loss $\mathcal{L}_{\mathrm{LV}}$ of \eqref{eq:lossw}. Every partial matrix reached in a rollout is itself an instance, so the gradient of each rollout reaches the policy at every reduced margin it visits.
At deployment we run SIS with MarginFlow as the proposal, one forward pass per row and no training on the new margins. One checkpoint serves every experiment.

\section{Experiments}
\label{sec:exp}

\subsection{Setup}
\label{sec:setup}

We train one network on a pool of margins and evaluate it zero-shot on held-out margins against the analytically designed proposals, with the details in Appendix~\ref{app:setup}. The pool holds 1904 margins, 1688 synthetic and 216 real. The synthetic margins come from six families that vary the shape, the density, and the unevenness of the margins, the three things that set an analytically designed proposal's distance from $q^*$ \citep{harrison2013importance}. They span 4 to 61 columns, row-to-column ratios from 1 to 140, and density 0.02 to 0.70. The real margins are species-by-site matrices \citep{atmar1995nestedness}, the interaction networks of Web of Life \citep{fortuna2014web}, item-response tables \citep{rizopoulos2006ltm}, and affiliation networks \citep{peixoto2020netzschleuder,davis1941deep}.

The test set holds 1190 margins held out from the pool by seed and, for the real collections, by source. They run from $3\times3$ to $870\times6$. Of these, 880 are synthetic draws with their own seed, 100 interpolate the family parameters between training values, and 210 are real tables (88 species-by-site matrices, 117 Web of Life networks, and 5 psychometric and social tables). The network enumerates the feasible row types at every state, and we train it where that enumeration stays under $2\times10^4$ types per state and evaluate it under $10^5$. Appendix~\ref{app:cost} reports the time per draw.

We report the effective sample fraction $\mathrm{ESS}/N$ of \eqref{eq:ess} as a percentage, so that 100\% is the zero-variance proposal. The hardest tier of Table~\ref{tab:main}, under 37\%, is where the post-hoc best loses more than one nat, since $e^{-1}=36.8\%$ in \eqref{eq:essd2}. Where the type-level state graph of Proposition~\ref{prop:reduced} has at most $4\times10^5$ states, dynamic programming evaluates the limit $e^{-D_2(p\,\|\,q)}$ of \eqref{eq:essd2} exactly, for MarginFlow and every baseline, since all of them treat columns of equal reduced sum alike. This is the case on 681 of the test margins. On the other 509 we estimate it from $N=4000$ draws and report the median over 16 independent repetitions. The reported value of MarginFlow on each test margin is the median over three training seeds, each trained for 25 hours on eight A100s. 

Our baselines are 31 configurations of analytically designed proposals, five proposals at six exponents $u\in[0.5,2]$ and the uniform distribution, all over the Gale--Ryser feasible rows. The five are the conditional Poisson proposal of \citet{chen2005sequential} with two weightings, the two proposals of \citet{harrison2013importance}, and the maximum-entropy proposal of \citet{glasserman2023maximum}, each with weights raised to power $u$. We report three of them. CDHL is the conditional Poisson proposal at $u=1$, the default of the networksis package \citep{admiraal2008networksis} and the one in widest use. Harrison--Miller is their proposal built on \citet{canfield2008asymptotic} at $u=1$, among the strongest configurations in the sweep and also the network before training. The post-hoc best is the best of all 31 on each margin, selected on 8 repetitions separate from the 16 reported, so every baseline stands on 24 repetitions per margin and the sweep takes about 9,000 hours of compute.
\subsection{Main results}
\label{sec:main}

\begin{table}[!htbp]
\caption{Median effective sample fraction (\%) on the 1190 held-out margins, by source and by the post-hoc best of the 31 configurations. The Harrison--Miller column is the network before training. W/T/L counts margins where MarginFlow is above, within 0.02 nats of, or below the post-hoc best.}
\label{tab:main}
\centering
\small
\setlength{\tabcolsep}{5pt}
\begin{tabular}{lcccccc}
\toprule
Margins & $n$ & CDHL & \shortstack{Harrison--\\Miller} & \shortstack{Post-hoc\\best} & \shortstack{MarginFlow\\(ours)} & W/T/L \\
\midrule
All & 1190 & 41.5 & 97.4 & 99.4 & 99.8 & 415/772/3 \\
\quad synthetic, new seeds & 880 & 34.8 & 97.2 & 99.4 & 99.8 & 316/563/1 \\
\quad synthetic, interpolated & 100 & 60.2 & 98.5 & 99.7 & 99.9 & 28/72/0 \\
\quad real tables & 210 & 53.0 & 97.5 & 99.2 & 99.8 & 71/137/2 \\
\midrule
Post-hoc best above 98\% & 755 & 56.6 & 98.9 & 99.9 & 99.9 & 0/754/1 \\
\quad between 90 and 98\% & 182 & 42.6 & 93.8 & 95.9 & 99.7 & 164/18/0 \\
\quad between 37 and 90\% & 197 & 16.4 & 71.4 & 76.6 & 99.0 & 195/0/2 \\
\quad below 37\% & 56 & 0.5 & 7.7 & 10.3 & 94.1 & 56/0/0 \\
\bottomrule
\end{tabular}
\end{table}

Table~\ref{tab:main} reports the median effective sample fraction against three baselines that stand for three users of classical SIS, one who runs the default, one who follows the literature to the strongest fixed proposal, and one who knew in advance which of the 31 configurations each margin needs, a choice no user can make. The default wastes three draws in five on a typical margin, and the literature recovers most of that loss. MarginFlow meets each test margin only at sampling time. It beats its untrained start on more than half of the margins and loses on none, and it matches or beats the hindsight choice on all but three, winning on a third of the real tables and losing on two. The medians in the upper block of the table sit close together because most held-out margins are easy for every proposal alike, and the gap lies in the tail, where MarginFlow raises the tenth percentile of the margins from the post-hoc best's 68.5\% to 98.1\%.

\Needspace{20\baselineskip}
\begin{wrapfigure}{r}{0.44\linewidth}
\centering
\includegraphics[width=0.97\linewidth]{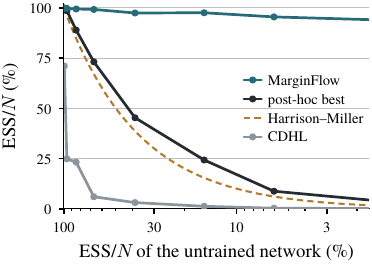}
\caption{Median effective sample fraction on the 1190 held-out margins, binned by the fraction of the untrained network, the dashed line.}
\label{fig:difficulty}
\end{wrapfigure}
The lower block of Table~\ref{tab:main} splits the margins by how well the post-hoc best does. Where that best already keeps nearly all of its draws, there is little left to gain, and the two tie on all but one margin. Below that MarginFlow pulls away, and once the post-hoc best falls under 90\% it wins all but two of the margins. Figure~\ref{fig:difficulty} sorts the margins instead by how far the untrained start is from $q^*$. As it collapses, the choice among the 31 configurations recovers only a few points, while training lifts the same margins back to a median above 94\% in every bin. The hardest tier consists mostly of tall dense margins, with up to 840 rows, from the power-law, extreme-sum, and bimodal families and from the held-out ecological and Web of Life collections alike. There the default and the untrained start keep almost nothing, choosing proposal and exponent with hindsight barely helps, and MarginFlow still keeps 94\% of its draws. Where analytical design and the choice among its products both give out, the learned proposal holds, on margins it has never seen.

\subsection{Error compounds over the rows}
\label{sec:rows}

\begin{wrapfigure}{r}{0.44\linewidth}
\centering
\includegraphics[width=0.97\linewidth]{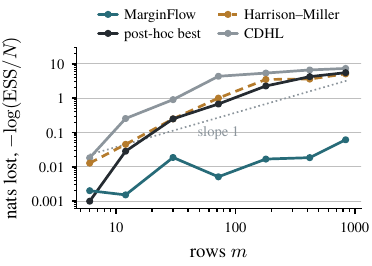}
\caption{Nats lost against the number of rows on power-law margins with six columns in the dense band. The dotted line has slope one.}
\label{fig:rows_main}
\end{wrapfigure}

Analytically designed proposals lose most on tall tables, and MarginFlow holds there because it errs less on every row. The weight of a matrix is a product over its rows, so a proposal that errs by $\delta$ nats on every row loses up to $m\delta^2$ nats in total, as Appendix~\ref{app:rows} proves. Figure~\ref{fig:rows_main} fixes family, width and density band and follows the nats lost from 6 to 840 rows. Every analytically designed proposal climbs with a slope near one on the log-log axes, and the Harrison--Miller proposal falls from 98.7\% at 6 rows to 0.6\% at 840. MarginFlow climbs as well, from a far smaller error per row, and stays below 0.1 nat at 840 rows and below one nat in every group of the dense band (Appendix~\ref{app:rows}). On the family of \citet{bezakova2012negative}, it extrapolates to five times the rows it trained on (Appendix~\ref{app:bssv}). A fixed formula errs by a fixed amount per row, and choosing among 31 changes little, as the post-hoc best shows. Training makes every row more accurate, and a tall table has less to compound.

\subsection{Training without the count}
\label{sec:count}

Figure~\ref{fig:curves} confirms the claims of Section~\ref{sec:flow} on one held-out margin, an $88\times6$ Web of Life table, the exactly evaluated margin on which the untrained start is farthest from $q^*$. Three seeds train on it alone, and at every checkpoint we draw $N=4000$ matrices and compute the exact $D_2(p\,\|\,q_\theta)$.

\begin{figure}[!htbp]
\centering
\includegraphics[width=0.76\linewidth]{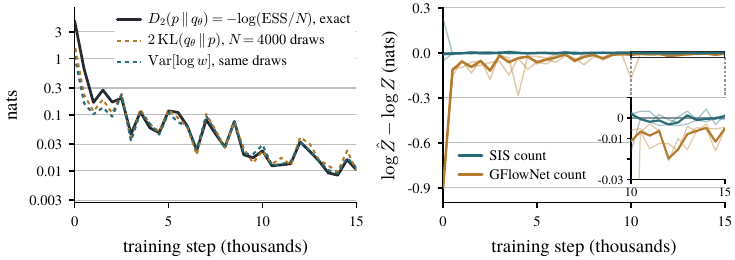}
\caption{Training on a held-out $88\times6$ Web of Life table, median over three seeds. Left, the exact divergence and the two quantities training sees on its own draws. Right, the same draws read as counts, by SIS through the mean weight $\log\bar{w}$ and by the GFlowNet through the mean log weight $\overline{\log w}$ that trajectory balance fits, with the seeds in light colour and the last third of training enlarged.}
\label{fig:curves}
\end{figure}

The left panel shows the loss, the variance of $\log w$ over the batch, falling from 1.1 nats to 0.01 in fifteen thousand steps. Twice the divergence $\KL(q_\theta\,\|\,p)$ falls with it. The exact $D_2(p\,\|\,q_\theta)$, which sets the quality of SIS, falls with both, from 4.7 nats to 0.01. At the start the three differ by a factor of four, and from about a thousand steps on they coincide, as Proposition~\ref{prop:ess} says they must once $q_\theta$ is close to $p$. Training on the GFlowNet loss thus trains the SIS proposal without ever using $\log Z$.
\FloatBarrier

The right panel reads the same draws as counts. The SIS estimate, the log of the mean weight, stays within 0.02 nats of $\log Z$ from the first thousand steps on, while the policy is far from uniform. The mean log weight, which trajectory balance would fit as $\log Z_\theta$, starts 0.9 nats below $\log Z$ and closes the gap only as the divergence closes. The inset shows it still short at the end of training, by the 0.005 nats of divergence that remain. The count therefore comes from SIS, as Remark~\ref{rem:count} says.

\section{Related work}
\label{sec:related}

SIS for fixed margins starts with \citet{snijders1991enumeration}, and every proposal since, from the conditional Poisson of \citet{chen2005sequential} to the maximum entropy of \citet{glasserman2023maximum}, is a closed form fixed in advance \citep{blanchet2009efficient,blitzstein2011sequential,harrison2013importance}. \citet{bezakova2012negative} exhibit margins on which the conditional Poisson proposal needs exponentially many draws. MarginFlow learns one proposal and serves unseen margins zero-shot.

Learned proposals predate GFlowNets \citep{gu2015neural,muller2019neural,wu2019solving,nicoli2020asymptotically}, each for one model, and \citet{zhao2024twisted} and \citet{choi2026reinforced} learn the twist and the proposal kernel of sequential Monte Carlo. Inference compilation \citep{paige2016inference,le2017inference} and conditional GFlowNets \citep{zhang2023robust,kim2025gfacs} amortize across instances, and we put this problem in that form through its self-similarity. Appendix~\ref{app:related} expands on each.

\section{Conclusion}
\label{sec:conclusion}

The proposal that sequential importance sampling has searched for since \citet{snijders1991enumeration} is the policy of a GFlowNet with unit reward on every matrix that has the given margins, and the count is its total flow. The policy is learned from its own draws, without the count, and we exploit the problem's self-similarity to learn it with one network for all margins. Trained once on 1904 margins, MarginFlow runs zero-shot on 1190 held-out margins and matches or beats the best of 31 analytically designed configurations on all but three, with the largest gains where analytical design gives out. Contingency tables with integer entries and graphs with prescribed degrees are also built row by row from a remainder that is again an instance, and the same construction carries over to each of them.

\bibliography{references}

@misc{fu2026spectral,
  author        = {Fu, Weibo and Qin, Qian and Wang, Guanyang},
  title         = {Spectral Gap for the Binary Fixed-Margin Swap Chain},
  year          = {2026},
  howpublished  = {arXiv:2606.22636},
  eprint        = {2606.22636},
  archivePrefix = {arXiv},
  primaryClass  = {math.PR},
  url           = {https://arxiv.org/abs/2606.22636}
}

@misc{nie2026snake,
  author        = {Nie, Zipei and Wang, Guanyang and Zhang, Peng},
  title         = {The {Snake} Algorithm: A Rejection-Free Sampler for Binary Matrices with Fixed Margins},
  year          = {2026},
  howpublished  = {arXiv:2608.17531},
  eprint        = {2608.17531},
  archivePrefix = {arXiv},
  primaryClass  = {stat.CO},
  url           = {https://arxiv.org/abs/2608.17531}
}

@article{chen2005sequential,
  title   = {Sequential {M}onte {C}arlo Methods for Statistical Analysis of Tables},
  author  = {Chen, Yuguo and Diaconis, Persi and Holmes, Susan P. and Liu, Jun S.},
  journal = {Journal of the American Statistical Association},
  volume  = {100},
  number  = {469},
  pages   = {109--120},
  year    = {2005},
  doi     = {10.1198/016214504000001303}
}

@article{miller2013exact,
  title   = {Exact Sampling and Counting for Fixed-Margin Matrices},
  author  = {Miller, Jeffrey W. and Harrison, Matthew T.},
  journal = {The Annals of Statistics},
  volume  = {41},
  number  = {3},
  pages   = {1569--1592},
  year    = {2013},
  doi     = {10.1214/13-AOS1131}
}

@article{snijders1991enumeration,
  title   = {Enumeration and Simulation Methods for 0--1 Matrices with Given Marginals},
  author  = {Snijders, Tom A. B.},
  journal = {Psychometrika},
  volume  = {56},
  number  = {3},
  pages   = {397--417},
  year    = {1991},
  doi     = {10.1007/BF02294482}
}

@article{connor1979assembly,
  title   = {The Assembly of Species Communities: Chance or Competition?},
  author  = {Connor, Edward F. and Simberloff, Daniel},
  journal = {Ecology},
  volume  = {60},
  number  = {6},
  pages   = {1132--1140},
  year    = {1979},
  doi     = {10.2307/1936961}
}

@article{kannan1999simple,
  title   = {Simple {M}arkov-Chain Algorithms for Generating Bipartite Graphs and Tournaments},
  author  = {Kannan, Ravi and Tetali, Prasad and Vempala, Santosh},
  journal = {Random Structures \& Algorithms},
  volume  = {14},
  number  = {4},
  pages   = {293--308},
  year    = {1999},
  doi     = {10.1002/(SICI)1098-2418(199907)14:4<293::AID-RSA1>3.0.CO;2-G}
}

@book{rasch1960probabilistic,
  title     = {Probabilistic Models for Some Intelligence and Attainment Tests},
  author    = {Rasch, Georg},
  publisher = {Danish Institute for Educational Research},
  address   = {Copenhagen},
  year      = {1960}
}

@article{chen2005exact,
  title   = {Exact Tests for the {R}asch Model via Sequential Importance Sampling},
  author  = {Chen, Yuguo and Small, Dylan},
  journal = {Psychometrika},
  volume  = {70},
  number  = {1},
  pages   = {11--30},
  year    = {2005},
  doi     = {10.1007/s11336-003-1069-1}
}

@article{blanchet2009efficient,
  title   = {Efficient Importance Sampling for Binary Contingency Tables},
  author = {Blanchet, Jos{\'e} H.},
  journal = {The Annals of Applied Probability},
  volume  = {19},
  number  = {3},
  year    = {2009},
  doi     = {10.1214/08-AAP558},
  pages = {949--982}
}

@article{harrison2013importance,
  title   = {Importance Sampling for Weighted Binary Random Matrices with Specified Margins},
  author  = {Harrison, Matthew T. and Miller, Jeffrey W.},
  journal = {arXiv preprint arXiv:1301.3928},
  year    = {2013}
}

@article{glasserman2023maximum,
  title   = {Maximum Entropy Distributions with Applications to Graph Simulation},
  author  = {Glasserman, Paul and Lelo de Larrea, Enrique},
  journal = {Operations Research},
  volume  = {71},
  number  = {5},
  pages   = {1908--1924},
  year    = {2023},
  doi     = {10.1287/opre.2022.2323}
}

@article{blitzstein2011sequential,
  title   = {A Sequential Importance Sampling Algorithm for Generating Random Graphs with Prescribed Degrees},
  author  = {Blitzstein, Joseph and Diaconis, Persi},
  journal = {Internet Mathematics},
  volume  = {6},
  number  = {4},
  pages   = {489--522},
  year    = {2011},
  doi     = {10.1080/15427951.2010.557277}
}

@article{bezakova2012negative,
  title   = {Negative Examples for Sequential Importance Sampling of Binary Contingency Tables},
  author  = {Bez{\'a}kov{\'a}, Ivona and Sinclair, Alistair and {\v{S}}tefankovi{\v{c}}, Daniel and Vigoda, Eric},
  journal = {Algorithmica},
  volume  = {64},
  number  = {4},
  pages   = {606--620},
  year    = {2012},
  doi     = {10.1007/s00453-011-9569-3}
}

@article{gotelli2012statistical,
  title   = {Statistical Challenges in Null Model Analysis},
  author  = {Gotelli, Nicholas J. and Ulrich, Werner},
  journal = {Oikos},
  volume  = {121},
  number  = {2},
  pages   = {171--180},
  year    = {2012},
  doi     = {10.1111/j.1600-0706.2011.20301.x}
}

@article{strona2014fast,
  title   = {A Fast and Unbiased Procedure to Randomize Ecological Binary Matrices with Fixed Row and Column Totals},
  author  = {Strona, Giovanni and Nappo, Domenico and Boccacci, Francesco and Fattorini, Simone and San-Miguel-Ayanz, Jesus},
  journal = {Nature Communications},
  volume  = {5},
  pages   = {4114},
  year    = {2014},
  doi     = {10.1038/ncomms5114}
}

@article{verhelst2008efficient,
  title   = {An Efficient {MCMC} Algorithm to Sample Binary Matrices with Fixed Marginals},
  author  = {Verhelst, Norman D.},
  journal = {Psychometrika},
  volume  = {73},
  number  = {4},
  pages   = {705--728},
  year    = {2008},
  doi     = {10.1007/s11336-008-9062-3}
}

@article{fosdick2018configuring,
  title   = {Configuring Random Graph Models with Fixed Degree Sequences},
  author  = {Fosdick, Bailey K. and Larremore, Daniel B. and Nishimura, Joel and Ugander, Johan},
  journal = {SIAM Review},
  volume  = {60},
  number  = {2},
  pages   = {315--355},
  year    = {2018},
  doi     = {10.1137/16M1087175}
}

@article{gobbi2014fast,
  title   = {Fast Randomization of Large Genomic Datasets While Preserving Alteration Counts},
  author  = {Gobbi, Andrea and Iorio, Francesco and Dawson, Kevin J. and Wedge, David C. and Tamborero, David and Alexandrov, Ludmil B. and Lopez-Bigas, Nuria and Garnett, Mathew J. and Jurman, Giuseppe and Saez-Rodriguez, Julio},
  journal = {Bioinformatics},
  volume  = {30},
  number  = {17},
  pages   = {i617--i623},
  year    = {2014},
  doi     = {10.1093/bioinformatics/btu474}
}

@article{erdos2022mixing,
  title   = {The Mixing Time of Switch {M}arkov Chains: A Unified Approach},
  author  = {Erd{\H{o}}s, P{\'e}ter L. and Greenhill, Catherine and Mezei, Tam{\'a}s R{\'o}bert and Mikl{\'o}s, Istv{\'a}n and Solt{\'e}sz, D{\'a}niel and Soukup, Lajos},
  journal = {European Journal of Combinatorics},
  volume  = {99},
  pages   = {103421},
  year    = {2022},
  doi     = {10.1016/j.ejc.2021.103421}
}

@inproceedings{bengio2021flow,
  title     = {Flow Network based Generative Models for Non-Iterative Diverse Candidate Generation},
  author    = {Bengio, Emmanuel and Jain, Moksh and Korablyov, Maksym and Precup, Doina and Bengio, Yoshua},
  booktitle = {Advances in Neural Information Processing Systems},
  volume    = {34},
  year      = {2021},
  note      = {arXiv:2106.04399},
  pages = {27381--27394}
}

@article{bengio2023gflownet,
  title   = {{GFlowNet} Foundations},
  author  = {Bengio, Yoshua and Lahlou, Salem and Deleu, Tristan and Hu, Edward J. and Tiwari, Mo and Bengio, Emmanuel},
  journal = {Journal of Machine Learning Research},
  volume  = {24},
  number  = {210},
  pages   = {1--55},
  year    = {2023},
  note    = {arXiv:2111.09266}
}

@inproceedings{whitammer2022trajectory,
  title     = {Trajectory Balance: Improved Credit Assignment in {GFlowNets}},
  author    = {Whitammer, Esmeralda S. and Jain, Moksh and Bengio, Emmanuel and Sun, Chen and Bengio, Yoshua},
  booktitle = {Advances in Neural Information Processing Systems},
  volume    = {35},
  year      = {2022},
  note      = {arXiv:2201.13259}
}

@inproceedings{richter2020vargrad,
  title     = {{VarGrad}: A Low-Variance Gradient Estimator for Variational Inference},
  author = {Richter, Lorenz and Boustati, Ayman and N{\"u}sken, Nikolas and Ruiz, Francisco J. R. and Akyildiz, {\"O}mer Deniz},
  booktitle = {Advances in Neural Information Processing Systems},
  volume    = {33},
  year      = {2020},
  note      = {arXiv:2010.10436},
  pages = {13481--13492}
}

@inproceedings{zhang2023robust,
  title     = {Robust Scheduling with {GFlowNets}},
  author    = {Zhang, David W. and Rainone, Corrado and Peschl, Markus and Bondesan, Roberto},
  booktitle = {International Conference on Learning Representations},
  year      = {2023},
  note      = {arXiv:2302.05446}
}

@inproceedings{whitammer2023gflownets,
  title     = {{GFlowNets} and Variational Inference},
  author    = {Whitammer, Esmeralda S. and Lahlou, Salem and Deleu, Tristan and Ji, Xu and Hu, Edward and Everett, Katie and Zhang, Dinghuai and Bengio, Yoshua},
  booktitle = {International Conference on Learning Representations},
  year      = {2023},
  note      = {arXiv:2210.00580}
}

@article{kong1994sequential,
  title   = {Sequential Imputations and {Bayesian} Missing Data Problems},
  author  = {Kong, Augustine and Liu, Jun S. and Wong, Wing Hung},
  journal = {Journal of the American Statistical Association},
  volume  = {89},
  number  = {425},
  pages   = {278--288},
  year    = {1994},
  doi = {10.1080/01621459.1994.10476469}
}

@article{agapiou2017importance,
  title   = {Importance Sampling: Intrinsic Dimension and Computational Cost},
  author  = {Agapiou, Sergios and Papaspiliopoulos, Omiros and Sanz-Alonso, Daniel and Stuart, Andrew M.},
  journal = {Statistical Science},
  volume  = {32},
  number  = {3},
  year    = {2017},
  note    = {arXiv:1511.06196},
  pages = {405--431},
  doi = {10.1214/17-STS611}
}

@inproceedings{lee2019set,
  title={Set transformer: A framework for attention-based permutation-invariant neural networks},
  author={Lee, Juho and Lee, Yoonho and Kim, Jungtaek and Kosiorek, Adam R. and Choi, Seungjin and Teh, Yee Whye},
  booktitle={Proceedings of the 36th International Conference on Machine Learning},
  year={2019},
  volume = {97},
  series = {Proceedings of Machine Learning Research},
  pages = {3744--3753},
  publisher = {PMLR}
}

@article{canfield2008asymptotic,
  title={Asymptotic enumeration of dense 0--1 matrices with specified line sums},
  author={Canfield, E. Rodney and Greenhill, Catherine and McKay, Brendan D.},
  journal={Journal of Combinatorial Theory, Series A},
  volume={115},
  number={1},
  pages={32--66},
  year={2008}
}

@misc{atmar1995nestedness,
  title={The nestedness temperature calculator: a Visual Basic program, including 294 presence-absence matrices},
  author={Atmar, Wirt and Patterson, Bruce D.},
  howpublished={AICS Research, Inc., University Park, NM, and The Field Museum, Chicago},
  year={1995}
}

@article{fortuna2014web,
  title={The Web of Life},
  author={Fortuna, Miguel A. and Ortega, Ra{\'u}l and Bascompte, Jordi},
  journal={arXiv preprint arXiv:1403.2575},
  year={2014}
}

@article{greenhill2006asymptotic,
  title={Asymptotic enumeration of sparse 0--1 matrices with irregular row and column sums},
  author={Greenhill, Catherine and McKay, Brendan D. and Wang, Xiaoji},
  journal={Journal of Combinatorial Theory, Series A},
  volume={113},
  number={2},
  pages={291--324},
  year={2006},
  doi = {10.1016/j.jcta.2005.03.005}
}

@article{rizopoulos2006ltm,
  title = {ltm: An {R} package for latent variable modeling and item response theory analyses},
  author={Rizopoulos, Dimitris},
  journal={Journal of Statistical Software},
  volume={17},
  number={5},
  pages={1--25},
  year={2006},
  doi = {10.18637/jss.v017.i05}
}

@misc{peixoto2020netzschleuder,
  title={The {N}etzschleuder network catalogue and repository},
  author={Peixoto, Tiago P.},
  howpublished={\url{https://networks.skewed.de/}},
  year={2020}
}

@book{davis1941deep,
  title={Deep South: A Social Anthropological Study of Caste and Class},
  author={Davis, Allison and Gardner, Burleigh B. and Gardner, Mary R.},
  publisher={University of Chicago Press},
  address={Chicago},
  year={1941}
}

@article{admiraal2008networksis,
  title={networksis: A package to simulate bipartite graphs with fixed marginals through sequential importance sampling},
  author={Admiraal, Ryan and Handcock, Mark S.},
  journal={Journal of Statistical Software},
  volume={24},
  number={8},
  pages={1--21},
  year={2008},
  doi = {10.18637/jss.v024.i08}
}

@article{jerrum1986random,
  title   = {Random Generation of Combinatorial Structures from a Uniform Distribution},
  author  = {Jerrum, Mark R. and Valiant, Leslie G. and Vazirani, Vijay V.},
  journal = {Theoretical Computer Science},
  volume  = {43},
  pages   = {169--188},
  year    = {1986},
  doi     = {10.1016/0304-3975(86)90174-X}
}

@inproceedings{gu2015neural,
  author    = {Shixiang Gu and Zoubin Ghahramani and Richard E. Turner},
  title     = {Neural Adaptive Sequential {Monte Carlo}},
  booktitle = {Advances in Neural Information Processing Systems},
  volume    = {28},
  year      = {2015},
  pages = {2629--2637}
}

@inproceedings{paige2016inference,
  author    = {Brooks Paige and Frank Wood},
  title     = {Inference Networks for Sequential {Monte Carlo} in Graphical Models},
  booktitle = {Proceedings of the 33rd International Conference on Machine Learning},
  year      = {2016},
  volume = {48},
  series = {Proceedings of Machine Learning Research},
  pages = {3040--3049},
  publisher = {PMLR}
}

@inproceedings{le2017inference,
  author    = {Tuan Anh Le and At{\i}l{\i}m G{\"u}ne{\c{s}} Baydin and Frank Wood},
  title     = {Inference Compilation and Universal Probabilistic Programming},
  booktitle = {Proceedings of the 20th International Conference on Artificial Intelligence and Statistics},
  year      = {2017},
  volume = {54},
  series = {Proceedings of Machine Learning Research},
  pages = {1338--1348},
  publisher = {PMLR}
}

@article{muller2019neural,
  author  = {Thomas M{\"u}ller and Brian McWilliams and Fabrice Rousselle and Markus Gross and Jan Nov{\'a}k},
  title   = {Neural Importance Sampling},
  journal = {ACM Transactions on Graphics},
  volume  = {38},
  number  = {5},
  pages   = {145:1--145:19},
  year    = {2019},
  doi = {10.1145/3341156}
}

@article{wu2019solving,
  author  = {Dian Wu and Lei Wang and Pan Zhang},
  title   = {Solving Statistical Mechanics Using Variational Autoregressive Networks},
  journal = {Physical Review Letters},
  volume  = {122},
  pages   = {080602},
  year    = {2019},
  number = {8},
  doi = {10.1103/PhysRevLett.122.080602}
}

@article{nicoli2020asymptotically,
  author  = {Kim A. Nicoli and Shinichi Nakajima and Nils Strodthoff and Wojciech Samek and Klaus-Robert M{\"u}ller and Pan Kessel},
  title   = {Asymptotically Unbiased Estimation of Physical Observables with Neural Samplers},
  journal = {Physical Review E},
  volume  = {101},
  pages   = {023304},
  year    = {2020},
  doi = {10.1103/PhysRevE.101.023304}
}

@inproceedings{choi2026reinforced,
  title     = {Reinforced Sequential {Monte Carlo} for Amortised Sampling},
  author    = {Choi, Sanghyeok and Mittal, Sarthak and Elvira, V{\'\i}ctor and Park, Jinkyoo and Whitammer, Esmeralda S.},
  booktitle = {International Conference on Machine Learning},
  year      = {2026},
  note      = {arXiv:2510.11711}
}

@inproceedings{boussif2025action,
  title     = {Action Abstractions for Amortized Sampling},
  author    = {Boussif, Oussama and Ezzine, L{\'e}na N{\'e}hale and Viviano, Joseph D. and Koziarski, Micha{\l} and Jain, Moksh and Whitammer, Esmeralda S. and Bengio, Emmanuel and Assouel, Rim and Bengio, Yoshua},
  booktitle = {International Conference on Learning Representations},
  year      = {2025},
  note      = {arXiv:2410.15184}
}

@article{neal2024pattern,
  title   = {Pattern Detection in Bipartite Networks: A Review of Terminology, Applications, and Methods},
  author  = {Neal, Zachary P. and Cadieux, Annabell and Garlaschelli, Diego and Gotelli, Nicholas J. and Saracco, Fabio and Squartini, Tiziano and Shutters, Shade T. and Ulrich, Werner and Wang, Guanyang and Strona, Giovanni},
  journal = {PLOS Complex Systems},
  volume  = {1},
  number  = {2},
  pages   = {e0000010},
  year    = {2024},
  doi     = {10.1371/journal.pcsy.0000010}
}

@article{draxler2025testing,
  title   = {Testing Measurement Invariance in a Conditional Likelihood Framework by Considering Multiple Covariates Simultaneously},
  author  = {Draxler, Clemens and Kurz, Andreas},
  journal = {Behavior Research Methods},
  volume  = {57},
  pages   = {50},
  year    = {2025},
  doi     = {10.3758/s13428-024-02551-9}
}

@article{neal2025stopping,
  title   = {A Stopping Rule for Randomly Sampling Bipartite Networks with Fixed Degree Sequences},
  author  = {Neal, Zachary P.},
  journal = {Social Networks},
  volume  = {80},
  pages   = {59--64},
  year    = {2025},
  doi     = {10.1016/j.socnet.2024.09.001}
}

@inproceedings{sun2025efficient,
  title     = {An Efficient Bipartite Graph Sampling Algorithm with Prescribed Degree Sequences},
  author    = {Sun, Tong and Hao, Jianshu and Zhang, Zhiyang and Jiang, Guangxin},
  booktitle = {Winter Simulation Conference},
  pages     = {259--270},
  year      = {2025},
  doi       = {10.1109/WSC68292.2025.11338997}
}

@inproceedings{dasilva2025right,
  title     = {When Do {GFlowNets} Learn the Right Distribution?},
  author    = {da Silva, Tiago and Alves, Rodrigo Barreto and de Souza da Silva, Eliezer and Souza, Amauri and Garg, Vikas and Kaski, Samuel and Mesquita, Diego},
  booktitle = {International Conference on Learning Representations},
  year      = {2025}
}

@inproceedings{fawkes2026ftb,
  title     = {$f$-Trajectory Balance: A Loss Family for Tuning {GFlowNets}, Generative Models, and {LLMs} with Off- and On-Policy Data},
  author    = {Fawkes, Jake and Hartford, Jason},
  booktitle = {International Conference on Machine Learning},
  year      = {2026},
  note      = {arXiv:2605.15417}
}

@inproceedings{tiapkin2024entropy,
  title     = {Generative Flow Networks as Entropy-Regularized {RL}},
  author    = {Tiapkin, Daniil and Morozov, Nikita and Naumov, Alexey and Vetrov, Dmitry},
  booktitle = {International Conference on Artificial Intelligence and Statistics},
  year      = {2024},
  note      = {arXiv:2310.12934}
}

@inproceedings{zhao2024twisted,
  title     = {Probabilistic Inference in Language Models via Twisted Sequential {Monte Carlo}},
  author    = {Zhao, Stephen and Brekelmans, Rob and Makhzani, Alireza and Grosse, Roger},
  booktitle = {International Conference on Machine Learning},
  year      = {2024},
  note      = {arXiv:2404.17546}
}

@inproceedings{kim2025gfacs,
  title     = {Ant Colony Sampling with {GFlowNets} for Combinatorial Optimization},
  author    = {Kim, Minsu and Choi, Sanghyeok and Kim, Hyeonah and Son, Jiwoo and Park, Jinkyoo and Bengio, Yoshua},
  booktitle = {International Conference on Artificial Intelligence and Statistics},
  year      = {2025},
  note      = {arXiv:2403.07041}
}

@article{jerdee2024improved,
  title   = {Improved Estimates for the Number of Non-Negative Integer Matrices with Given Row and Column Sums},
  author  = {Jerdee, Maximilian and Kirkley, Alec and Newman, M. E. J.},
  journal = {Proceedings of the Royal Society A},
  volume  = {480},
  number  = {2282},
  pages   = {20230470},
  year    = {2024},
  doi     = {10.1098/rspa.2023.0470}
}

@inproceedings{deleu2024control,
  title     = {Discrete Probabilistic Inference as Control in Multi-path Environments},
  author    = {Deleu, Tristan and Nouri, Padideh and Malkin, Nikolay and Precup, Doina and Bengio, Yoshua},
  booktitle = {Conference on Uncertainty in Artificial Intelligence},
  year      = {2024},
  note      = {arXiv:2402.10309}
}

@inproceedings{chen2026powerflow,
  title     = {{PowerFlow}: Unlocking the Dual Nature of {LLMs} via Principled Distribution Matching},
  author    = {Chen, Ruishuo and Chen, Yu and Li, Zhuoran and Huang, Longbo},
  booktitle = {International Conference on Machine Learning},
  year      = {2026},
  note      = {arXiv:2603.18363}
}
\bibliographystyle{dilab_ref}

\makeappendixtoc
\appendix
\section{Proofs}
\label{app:proofs}

Throughout, $p$ is the uniform distribution on $\Omega(r,c)$, $q$ is a policy over feasible rows with $q(X)=\prod_i q(x_i\mid X_{<i})$, and $w(X)=1/q(X)$.

\begin{proof}[Proof of Lemma~\ref{lem:flow}]
On a tree the flow through a state is the total reward of the terminal states below it \citep{bengio2021flow}. Below $X_{<i}$ lie exactly the completions of $X_{<i}$, each with reward one, so $F(X_{<i})=Z(X_{<i})$. The edge flow into $X_{\le i}$ is $F(X_{\le i})$, so $P_F(x_i\mid X_{<i})=Z(X_{\le i})/Z(X_{<i})$.
\end{proof}

\begin{proof}[Proof of Theorem~\ref{thm:sis}]
On a tree $P_B\equiv1$ and $R\equiv1$, so the trajectory-balance condition $Z_\theta P_F(\tau;\theta)=R(X)$ reads $Z_\theta q_\theta(X)=1$, and $\log R(X)-\log P_F(\tau;\theta)=-\log q_\theta(X)=\log w_\theta(X)$, which gives (i) and \eqref{eq:lossw}. For (iii), if the population variance of $\log w_\theta$ under $q_\theta$ is zero, then $q_\theta$ is constant on its support, and with full support that constant is $1/Z$. Then $q_\theta(X)=q^*(X)$ for every $X$, and since the tree has one path to each $X$, $q_\theta=q^*$ at every state. Conversely every weight under $q^*$ is $Z$.
\end{proof}

\begin{proof}[Proof of Theorem~\ref{thm:grad}]
Write $L_b=\log w_\theta(X_b)$, $\bar L$ for the batch mean, and $S_b=\nabla_\theta\log q_\theta(X_b)$, so that $\nabla_\theta L_b=-S_b$. With the draws held fixed and the variance normalized by $B$, $\nabla_\theta\mathcal{L}_{\mathrm{LV}}=\frac{2}{B}\sum_b(L_b-\bar L)\nabla_\theta L_b=-\frac{2}{B}\sum_b(L_b-\bar L)S_b$, since $\sum_b(L_b-\bar L)=0$. The draws are independent and $\mathbb{E}_{q_\theta}[S]=0$, so $\mathbb{E}[(L_b-\bar L)S_b]=\frac{B-1}{B}\mathbb{E}_{q_\theta}[LS]$. Meanwhile $\nabla_\theta\KL(q_\theta\,\|\,p)=\mathbb{E}_{q_\theta}[(\log q_\theta-\log p)S]=-\mathbb{E}_{q_\theta}[LS]$, because $\log p=-\log Z$ is constant and $\mathbb{E}_{q_\theta}[S]=0$. Combining, $\mathbb{E}[\nabla_\theta\mathcal{L}_{\mathrm{LV}}]=2\frac{B-1}{B}\nabla_\theta\KL(q_\theta\,\|\,p)$. Finally $\KL(q_\theta\,\|\,p)=\mathbb{E}_{q_\theta}[\log q_\theta]+\log Z=\log Z-H(q_\theta)$, and $\mathbb{E}_{q_\theta}[L]=-\mathbb{E}_{q_\theta}[\log q_\theta]=H(q_\theta)$.
\end{proof}

\begin{proof}[Proof of Proposition~\ref{prop:ess}]
By the law of large numbers, $\mathrm{ESS}/N=(\frac1N\sum_\ell w_\ell)^2/(\frac1N\sum_\ell w_\ell^2)\to(\mathbb{E}_q[w])^2/\mathbb{E}_q[w^2]=Z^2/\mathbb{E}_q[w^2]$. Here $\mathbb{E}_q[w^2]/Z^2=\sum_X q(X)/(q(X)^2Z^2)=\sum_X p(X)^2/q(X)=e^{D_2(p\|q)}$. The count estimate is a mean of $N$ independent weights with mean $Z$, so its relative mean squared error is $(\mathbb{E}_q[w^2]/Z^2-1)/N$. For the expansion, write $q=p(1+h)$ and expand to second order in $h$. First, $e^{D_2}=\mathbb{E}_p[1/(1+h)]=1+\mathbb{E}_p[h^2]+O(\|h\|_\infty^3)$. Second, $\log w=\log Z-\log(1+h)$, and under $q$, $\mathbb{E}_q[\log(1+h)]=\mathbb{E}_p[(1+h)(h-h^2/2)]+O(\|h\|_\infty^3)=\frac12\mathbb{E}_p[h^2]+O(\|h\|_\infty^3)$ and $\mathbb{E}_q[\log^2(1+h)]=\mathbb{E}_p[h^2]+O(\|h\|_\infty^3)$, so $\operatorname{Var}_q[\log w]=\mathbb{E}_p[h^2]+O(\|h\|_\infty^3)$. Third, $\KL(q\|p)=\mathbb{E}_p[(1+h)\log(1+h)]=\frac12\mathbb{E}_p[h^2]+O(\|h\|_\infty^3)$.
\end{proof}

\begin{proof}[Proof of Proposition~\ref{prop:reduced}]
The completions of $X_{<i}$ are the matrices in $\Omega(r_{\ge i},d)$, so $Z(X_{<i})=|\Omega(r_{\ge i},d)|$ and $Z(X_{\le i})=|\Omega(r_{>i},d-x_i)|$, and $q^*(x_i\mid X_{<i})$ is the ratio of the two. Permuting the columns of every matrix in $\Omega(r,c)$ by $\pi$ is a bijection onto $\Omega(r,\pi c)$, and permuting the rows is a bijection onto $\Omega(\sigma r,c)$. Both counts are therefore unchanged when $(x_i,d)$ is replaced by $(\pi x_i,\pi d)$ or $r_{>i}$ by $\sigma r_{>i}$. Two rows that place the same number of ones in the columns of each reduced sum differ by a permutation $\pi$ that fixes $d$, so they have the same probability.
\end{proof}

\section{Experimental details}
\label{app:setup}

\subsection{The pool}
\label{app:pool}

Shapes in the pool run from 4 to 61 columns with row-to-column ratios from 1 to 140. Densities run from 0.03 to 0.70 in three bands, sparse, medium, and dense, and down to 0.02 in the family of \citet{bezakova2012negative}. The margins come from six synthetic families, the largest at 840 rows and the smallest at 4.

Four families start from a propensity for each row and each column. Entry $(i,j)$ is then one with probability proportional to the product of the two propensities, scaled so that the expected density is the target, and the margins are the row and column sums of the realized matrix. Power-law propensities with exponent $\alpha\in\{0.6,1.0,1.4\}$ give heavy-tailed margins, and constant propensities give near-regular ones. A full row and a full column added to a power-law draw put one extreme sum on either side, and two or three activity levels on each side give bimodal margins. The other two families are built directly, without propensities. Columns filled to near saturation make the Gale--Ryser test rule out most rows at every step. The family of \citet{bezakova2012negative}, at up to $84\times61$, is the one on which the conditional Poisson proposal is proven to fail exponentially.

The real margins come from four collections. The 294 species-by-site matrices shipped with the nestedness temperature calculator \citep{atmar1995nestedness} and the interaction networks of Web of Life \citep{fortuna2014web} cover the ecological use. Item-response tables from the ltm package \citep{rizopoulos2006ltm} cover Rasch conditional inference, and affiliation networks from Netzschleuder \citep{peixoto2020netzschleuder} and the Southern Women data of \citet{davis1941deep} cover the social-network use.

\subsection{Splits}
\label{app:splits}

For every family, each combination of shape and density band is generated four times with different random seeds. Two of the four draws form the pool, one the development set, and one the test set. A further 100 test margins use power-law exponents and activity levels that lie between the training values, so they probe interpolation in the family parameters.

The real collections are split by their source. Each species-by-site matrix is one unit, the Web of Life networks that share a study form one unit, and each item-response table or affiliation network is one unit. A unit goes whole to the pool, the development set, or the test set, in the proportions 55/15/30 for the species-by-site matrices and 40/15/45 for Web of Life. This puts 126 species-by-site matrices, 88 Web of Life networks, and 2 psychometric and social tables in the pool, and leaves the others for the development and test sets. The type cap of Section~\ref{sec:setup} keeps 1904 of the 2391 training margins in the pool, 963 in the development set, and 1190 in the test set.

\subsection{Training and architecture}
\label{app:training}

We train three seeds of MarginFlow for 15,000 steps at learning rate $10^{-4}$, each step drawing four margins from the pool with $B=64$ matrices each. We pick the checkpoint by the exact effective sample fraction on the 963 development margins. All three seeds select the final step, and across the test margins their fractions differ by a median of 0.08 percentage points. Each seed trains for about 25 hours on a machine with eight A100s.

The architecture is chosen on a smaller gate before training. The two candidates are a Deep-Sets scorer that sums embeddings of the columns and the set transformer of Section~\ref{sec:network}. We fitted both to the exact $q^*$ by the per-state divergence $\KL(q^*\,\|\,q_\theta)$ on 34 small training margins and scored them on 8 held-out ones. At 4.0M parameters the Deep-Sets scorer reached an effective sample fraction of 43\%, and the set transformer reached 81\% at 1.7M, 85\% at 3.3M and 86\% at 12.9M, so we kept the 3.3M configuration, four layers of width 256. We also built a policy that places the ones of a row group by group, with a suffix dynamic program normalizing each step. It is exact as well, but it calls the network once per group instead of once per row, and its training step takes nine times as long.

\subsection{The analytically designed term of the logit}
\label{app:prior}

At row $i$, with $R=m-i+1$ rows still to place, $n$ columns, reduced column sums $d$, and remaining row sums $r_{i+1},\dots,r_m$ after the current row, let $T=\sum_{k>i}r_k$. The term $a(t)$ of \eqref{eq:logit} for a type $t=(c_v)_v$ is
\begin{equation}
a(t)=\sum_{0<v<R} c_v\,\alpha_v,\qquad
\alpha_v=\log\frac{v}{R-v}+\kappa\Bigl(\frac12-v+\frac{T}{n}\Bigr),\qquad
\kappa=\eta(1-\nu),
\end{equation}
with
\begin{equation}
\eta=\frac{n(R-1)}{T\,\bigl(n(R-1)-T\bigr)},\qquad
\nu=\eta\sum_{k>i}\Bigl(r_k-\frac{T}{R-1}\Bigr)^2,
\end{equation}
and $\kappa=0$ when $T=0$, $R=1$ or $T=n(R-1)$. Here $\alpha_v$ is the log weight of a one in a column with reduced sum $v$, the row-by-row form of the column weights that \citet{harrison2013importance} derive from the asymptotic enumeration formula of \citet{canfield2008asymptotic}. A column with $d_j=0$ or $d_j=R$ receives the same entry in every remaining row, so its term is the same for every feasible type and is left out of the sum.

\subsection{Baselines and evaluation}
\label{app:baselines}

The 31 baseline configurations come from five proposals at six exponents, and the uniform distribution over feasible rows. The five are the conditional Poisson proposal of \citet{chen2005sequential} with the weights $d_j/(R-d_j)$ and with those analysed by \citet{blanchet2009efficient}, the two proposals of \citet{harrison2013importance} built on the enumeration asymptotics of \citet{canfield2008asymptotic} and of \citet{greenhill2006asymptotic}, and the maximum-entropy proposal of \citet{glasserman2023maximum}. Each has its weights raised to a power $u\in\{0.5,0.75,1,1.25,1.5,2\}$. The two proposals of \citet{harrison2013importance} at $u=1$ are the strongest in the sweep. Every margin has rows as the longer side, sorted by decreasing row sum, as \citet{chen2005sequential} recommend, and every proposal builds it row by row.

Most of the compute goes into evaluation. The 31 configurations and the three seeds of MarginFlow are run on all 1190 test margins with 24 repetitions of $N=4000$ draws each, 8 for the post-hoc selection and 16 for the report. The Harrison--Miller column of Tables~\ref{tab:main} and~\ref{tab:hard} is the untrained network on its own draws. Where the post-hoc best selects that configuration, the two columns run the same proposal on separate draws and differ only by Monte Carlo noise. The gap is below 0.001 nats on every margin counted as a tie and reaches a factor of two in effective sample fraction on the hardest. Dynamic programming gave the exact fraction on the 681 margins whose type-level state graph has at most $4\times10^5$ states. Building the type graph and counting takes under a second on three quarters of these margins and up to ten minutes on the largest. Exact counting stops at the cap, and sequential importance sampling takes over. The 509 margins above it are also where MarginFlow's advantage is widest, with the tenth percentile of the post-hoc best at 35\% against 96\% for MarginFlow (Table~\ref{tab:family}). The evaluation takes about 9,000 hours of compute in total.

On the 509 margins above the cap the effective sample fraction is itself estimated from the draws, and a proposal that never visits a heavy region reports too high a fraction. The count estimate offers a check that needs no exact reference. Every proposal gives an unbiased estimate of $Z$, so a proposal that misses mass reports a lower $\log\hat Z$ than one that finds it. On 94\% of the 509 margins MarginFlow and the post-hoc best agree within Monte Carlo error, with a median gap below 0.001 nats. On the four margins where the gap exceeds 0.09 nats, all dense with 420 to 840 rows, the lower estimate is the post-hoc best. The three seeds of MarginFlow agree within 0.008 nats on every margin. The same holds against the whole grid of 31 configurations. On every one of the 509 margins, each analytically designed configuration whose own effective sample fraction is at least 20\% reports a count within 0.02 nats of MarginFlow's, and such a configuration exists on 475 of them.

\section{Additional experiments}
\label{app:more}

\subsection{The analytically designed term matters on the hardest margins}
\label{app:ablation}

The analytically designed term $a(t)$ of \eqref{eq:logit} makes a difference on the hardest margins. Table~\ref{tab:prior} trains the same network, with the same recipe and the same three seeds, with the term removed, and reports it beside MarginFlow on the 1183 test margins on which its evaluation finished. The seven missing are tall tables of 240 and 600 rows in the medium and dense bands and in the two hardest tiers. There the evaluation runs the network on every state that 16 repetitions of 4000 draws visit, and the network without the term spreads its draws over too many states to finish in time, so the seven it misses are its hardest margins and the table reads in its favour.

Without the term the untrained proposal is the uniform draw over feasible rows and keeps almost nothing, so the head has to learn the whole proposal rather than a correction. On the typical margin it does, and the two networks tie. On the hardest tier it still beats the post-hoc best on every margin, but keeps 82\% of its draws against 95\% with the term. Across all margins its three seeds disagree five times more. The term costs nothing at deployment and starts training from a proposal that already keeps 97\%. Without that start, the head falls behind on the hardest margins.

\begin{table}[htbp]
\caption{Median effective sample fraction (\%) with and without the analytically designed term $a(t)$, by the tiers of Table~\ref{tab:main}, on the 1183 test margins where both are evaluated. W/T/L compares the trained network without the term to the one with it.}
\label{tab:prior}
\centering
\small
\setlength{\tabcolsep}{4pt}
\begin{tabular}{lrcccccc}
\toprule
& & \multicolumn{2}{c}{Without $a(t)$} & \multicolumn{2}{c}{With $a(t)$} & & \\
\cmidrule(lr){3-4}\cmidrule(lr){5-6}
Margins & $n$ & untrained & trained & untrained & trained & Post-hoc best & W/T/L \\
\midrule
All & 1183 & 0.0 & 99.3 & 97.4 & 99.8 & 99.4 & 9/895/279 \\
\quad post-hoc best above 98\% & 755 & 0.1 & 99.7 & 98.9 & 99.9 & 99.9 & 0/684/71 \\
\quad between 90 and 98\% & 182 & 0.0 & 98.7 & 93.8 & 99.7 & 95.9 & 0/135/47 \\
\quad between 37 and 90\% & 196 & 0.0 & 95.4 & 71.7 & 99.0 & 76.8 & 5/75/116 \\
\quad below 37\% & 50 & 0.0 & 81.6 & 8.3 & 95.1 & 11.4 & 4/1/45 \\
\bottomrule
\end{tabular}
\end{table}

\subsection{MarginFlow compounds less error over the rows}
\label{app:rows}

The weight of a matrix is a product over its rows, and its variance collects one term per row. The term for row $i$ is the chi-square error of the policy at that row, weighted by the squared likelihood ratio of the prefix, and the bound that Section~\ref{sec:rows} uses follows from it.

\begin{theorem}[Row errors add up]
\label{thm:rows}
Let $q$ have full support on $\Omega(r,c)$. For a matrix $X$, write $\Lambda_i(X)=Z(X_{\le i})/(Z\,q(X_{\le i}))$ for the likelihood ratio of its first $i$ rows, so that $\Lambda_0=1$ and $\Lambda_m=w(X)/Z$, and write $\chi^2_i(X_{<i})=\sum_{x_i}q^*(x_i\mid X_{<i})^2/q(x_i\mid X_{<i})-1$ for the chi-square divergence from $q^*$ to $q$ at the state $X_{<i}$. Then
\begin{equation}
\frac{\mathbb{E}_q[w^2]}{Z^2}-1=\sum_{i=1}^m\mathbb{E}_q\bigl[\Lambda_{i-1}^2\,\chi^2_i(X_{<i})\bigr].
\label{eq:rows}
\end{equation}
\end{theorem}

\begin{corollary}
\label{cor:rows}
Suppose that at every state the log-probabilities of $q$ differ from those of $q^*$ by errors that lie in an interval of width $2\delta$. Then $D_2(p\,\|\,q)\le m\delta^2$.
\end{corollary}

\begin{proof}[Proof of Theorem~\ref{thm:rows} and Corollary~\ref{cor:rows}]
Along a trajectory, $\Lambda_i/\Lambda_{i-1}=Z(X_{\le i})/(Z(X_{<i})\,q(x_i\mid X_{<i}))=q^*(x_i\mid X_{<i})/q(x_i\mid X_{<i})$. Conditionally on $X_{<i}$ this ratio has mean $\sum_{x_i}q^*(x_i\mid X_{<i})=1$ and second moment $1+\chi^2_i(X_{<i})$, so $\Lambda_i$ is a martingale under $q$ and $\mathbb{E}_q[\Lambda_i^2]-\mathbb{E}_q[\Lambda_{i-1}^2]=\mathbb{E}_q[\Lambda_{i-1}^2\chi^2_i(X_{<i})]$. Summing from $\Lambda_0=1$ to $\Lambda_m=w/Z$ gives \eqref{eq:rows}. For the corollary, write $q(x_i\mid X_{<i})\propto q^*(x_i\mid X_{<i})e^{\epsilon(x_i)}$ with $\max\epsilon-\min\epsilon\le2\delta$ at the state. Then $1+\chi^2_i=\mathbb{E}_{q^*}[e^{\epsilon}]\,\mathbb{E}_{q^*}[e^{-\epsilon}]\le\cosh^2\delta$ by the Kantorovich inequality for a positive variable with range ratio at most $e^{2\delta}$. Let $A(X_{<i})=\mathbb{E}_q[(\Lambda_m/\Lambda_{i-1})^2\mid X_{<i}]$ be the second moment of the likelihood ratio of the rows from $i$ on, so that $A(X)=1$ at complete matrices and $A(X_{<i})=\sum_{x_i}q^*(x_i\mid X_{<i})^2/q(x_i\mid X_{<i})\,A(X_{\le i})\le\cosh^2\delta\,\max_{x_i}A(X_{\le i})$. Induction over the rows gives $e^{D_2(p\|q)}=A(X_{<1})\le\cosh^{2m}\delta$, and $\log\cosh\delta\le\delta^2/2$.
\end{proof}

A per-row error of $\delta$ nats costs at most $m\delta^2$ nats, so a fixed error per row compounds linearly, which is why analytically designed proposals lose fit on larger tables \citep{harrison2013importance,verhelst2008efficient}.

\begin{figure}[htbp]
\centering
\includegraphics[width=0.46\linewidth]{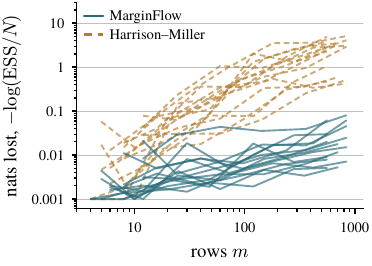}
\caption{Nats lost against the number of rows for the Harrison--Miller proposal and MarginFlow on all 15 groups of power-law margins in the dense band.}
\label{fig:rows}
\end{figure}

Figure~\ref{fig:rows} extends Figure~\ref{fig:rows_main} to all 15 groups of the dense band. In every group where the Harrison--Miller proposal loses at least half a nat, its loss grows with a slope near one. MarginFlow's stays under 0.1 nat in all but three groups, the widest tables of the extreme-sum and bimodal families, where it reaches 0.7 nat. The near-regular family shows no growth for any proposal, since every proposal is already close to $q^*$ on every row there.

\subsection{Extrapolation on the family where classical SIS fails}
\label{app:bssv}

\citet{bezakova2012negative} prove that on the margins $r=(1,\dots,1,\lfloor\beta m\rfloor)$ and $c=(1,\dots,1,\lfloor\gamma m\rfloor)$ with $m+1$ rows the conditional Poisson proposal underestimates the count by an exponential factor unless it is run for an exponential number of draws. The pool holds this family at $m\le60$, which is at most $84\times61$ after transposition. The six test margins take $m$ to 120, 200 and 300 at $(\beta,\gamma)=(0.5,0.25)$ and $(0.25,0.5)$, up to $376\times301$, so they ask MarginFlow to extrapolate to five times the $m$ it is trained on. Every state admits only two feasible row types, so the effective sample fraction is exact for every proposal. Table~\ref{tab:bssv} reports it as the number of draws per effective sample, $N/\mathrm{ESS}=e^{D_2(p\,\|\,q)}$, the price of one uniform matrix.

\begin{table}[htbp]
\caption{Draws per effective sample on the six test margins of the family of \citet{bezakova2012negative}, exact for every proposal. MarginFlow is the median of the three seeds with their range.}
\label{tab:bssv}
\centering
\small
\begin{tabular}{lccc}
\toprule
Margins & CDHL & Harrison--Miller & MarginFlow (ours) \\
\midrule
$121\times91$ & $1.0\times10^{11}$ & 10.5 & 1.02 (1.01--1.02) \\
$151\times121$ & $1.1\times10^{10}$ & 2.3 & 1.02 (1.01--1.08) \\
$201\times151$ & $1.9\times10^{19}$ & 116.7 & 1.06 (1.02--1.16) \\
$251\times201$ & $5.1\times10^{17}$ & 4.8 & 1.07 (1.01--1.51) \\
$301\times226$ & $4.5\times10^{29}$ & 3011 & 1.20 (1.08--1.53) \\
$376\times301$ & $2.3\times10^{27}$ & 13.6 & 1.21 (1.02--2.87) \\
\bottomrule
\end{tabular}
\end{table}

The theorem is visible in the CDHL column, which pays $e^{63}$ draws per effective sample at $376\times301$. The untrained start pays between 2 and 3000. MarginFlow, which never sees more than 84 rows of this family, pays 1.2 at $376\times301$ for its median seed, so the learned proposal extrapolates along the rows where the analytically designed one does not.

\subsection{Counting with MarginFlow against a Markov chain}
\label{app:count}

We compare MarginFlow as a counter with the MCMC route on 18 test margins whose count is known exactly, among them the Southern Women table of \citet{davis1941deep}, with log counts from 8 to 266. A Markov chain gives uniform samples and no count, so a counter has to be assembled from it. We use the self-reducibility argument of \citet{jerrum1986random}. Any one table has probability $1/|\Omega(r,c)|$ under the uniform distribution. Fixing its entries one at a time factors this probability into the probabilities that each entry takes its value given the entries fixed before it. Each factor is a marginal of the uniform distribution over the tables that agree with the fixed entries, and a Curveball chain \citep{strona2014fast} samples that distribution. The log count is the negative sum of the estimated log factors. We run it for 10 and for 60 minutes of one core with four seeds and report the standard deviation of its log count across the seeds. MarginFlow and CDHL draw $N=4000$ matrices, the network on one GPU and CDHL on one core, and Table~\ref{tab:count} reports the standard deviation of their log counts across 16 repetitions.

\begin{table}[htbp]
\caption{Standard deviation of the log count on 18 held-out margins with known counts, as median and range over the margins, and wall-clock time per margin.}
\label{tab:count}
\centering
\small
\begin{tabular}{lcc}
\toprule
Counter & Standard deviation of the log count & Time per margin \\
\midrule
MarginFlow, $N=4000$, one GPU & 0.0007 [0.0002, 0.0022] & 8.5 s \\
CDHL, $N=4000$, one core & 0.017 [0.004, 0.065] & 42 s \\
Curveball chain, one core & 0.038 [0.001, 0.081] & 600 s \\
Curveball chain, one core & 0.015 [0.002, 0.051] & 3600 s \\
\bottomrule
\end{tabular}
\end{table}

The hour-long chain shows no bias its standard error can detect, and that standard error is calibrated. Its log count still scatters about twenty times more than MarginFlow's, in 3600 seconds of one core against 8.5 seconds on a GPU. CDHL at the same $N$ is as precise as the hour-long chain in under a minute, which is the case for SIS that \citet{chen2005sequential} made. MarginFlow cuts the scatter of CDHL by a further factor of fifteen on the median margin of Table~\ref{tab:count_full}.

Table~\ref{tab:count_full} lists the 18 margins with the exact log count from the dynamic programming of \citet{miller2013exact}. For MarginFlow and CDHL it gives the mean error and standard deviation of the log count over 16 repetitions of $N=4000$ draws, and for the chain the same over four seeds at each budget. On all 18 margins MarginFlow's mean error is within two standard errors of zero.

\begin{table}[htbp]
\caption{Counting on the 18 margins with known counts. Mean error and standard deviation of the log count for MarginFlow and CDHL are over 16 repetitions of $N=4000$ draws, and for the chain over four seeds at 10 and at 60 minutes of one core.}
\label{tab:count_full}
\centering
\small
\setlength{\tabcolsep}{3pt}
\begin{tabular}{llrcccccccc}
\toprule
& & & \multicolumn{2}{c}{MarginFlow} & \multicolumn{2}{c}{CDHL} & \multicolumn{2}{c}{Chain, 10 min} & \multicolumn{2}{c}{Chain, 60 min} \\
\cmidrule(lr){4-5}\cmidrule(lr){6-7}\cmidrule(lr){8-9}\cmidrule(lr){10-11}
Margins & Size & $\log Z$ & error & sd & error & sd & error & sd & error & sd \\
\midrule
power law 0.6 & $6\times6$ & 8.4 & 0.0000 & 0.0003 & +0.0017 & 0.0039 & +0.001 & 0.001 & 0.000 & 0.002 \\
tight columns & $16\times14$ & 51.6 & 0.0000 & 0.0004 & +0.0010 & 0.0042 & +0.005 & 0.017 & -0.005 & 0.012 \\
species by site & $26\times12$ & 70.7 & +0.0001 & 0.0014 & +0.0091 & 0.0205 & +0.001 & 0.010 & -0.011 & 0.015 \\
bimodal & $24\times12$ & 81.9 & -0.0005 & 0.0017 & +0.0020 & 0.0179 & -0.015 & 0.025 & +0.007 & 0.013 \\
Southern Women & $18\times14$ & 85.2 & +0.0001 & 0.0006 & -0.0002 & 0.0067 & -0.017 & 0.046 & +0.004 & 0.015 \\
power law 0.6 & $24\times12$ & 88.3 & -0.0002 & 0.0008 & +0.0028 & 0.0084 & +0.035 & 0.025 & +0.004 & 0.011 \\
power law 1.2 & $72\times6$ & 94.7 & -0.0007 & 0.0016 & -0.0155 & 0.0612 & +0.022 & 0.035 & +0.018 & 0.007 \\
bimodal & $40\times8$ & 96.7 & +0.0002 & 0.0010 & +0.0007 & 0.0159 & +0.014 & 0.062 & +0.004 & 0.020 \\
species by site & $40\times12$ & 97.3 & -0.0003 & 0.0021 & -0.0009 & 0.0317 & 0.000 & 0.050 & +0.007 & 0.007 \\
bimodal & $20\times20$ & 101.0 & -0.0001 & 0.0022 & -0.0001 & 0.0317 & +0.036 & 0.047 & +0.011 & 0.018 \\
power law 1.0 & $72\times6$ & 106.3 & -0.0003 & 0.0011 & -0.0183 & 0.0653 & +0.015 & 0.033 & +0.002 & 0.005 \\
power law 1.4 & $40\times20$ & 118.3 & 0.0000 & 0.0004 & +0.0057 & 0.0126 & +0.006 & 0.038 & -0.017 & 0.018 \\
power law 0.6 & $40\times20$ & 119.5 & 0.0000 & 0.0003 & +0.0065 & 0.0071 & +0.032 & 0.023 & +0.003 & 0.025 \\
Web of Life & $39\times15$ & 119.9 & +0.0001 & 0.0009 & +0.0066 & 0.0323 & +0.026 & 0.070 & +0.001 & 0.033 \\
power law 0.8 & $72\times6$ & 123.0 & +0.0001 & 0.0004 & -0.0123 & 0.0317 & +0.048 & 0.081 & -0.020 & 0.007 \\
near-regular & $20\times20$ & 142.2 & 0.0000 & 0.0002 & -0.0002 & 0.0044 & +0.030 & 0.067 & -0.012 & 0.042 \\
tight columns & $26\times22$ & 149.0 & 0.0000 & 0.0005 & +0.0011 & 0.0040 & -0.004 & 0.038 & -0.011 & 0.051 \\
extreme sums & $560\times4$ & 266.3 & +0.0001 & 0.0006 & -0.0132 & 0.0257 & +0.001 & 0.058 & -0.012 & 0.032 \\
\bottomrule
\end{tabular}
\end{table}

The table whose probability the counter factors is not fixed in advance. The counter builds it entry by entry, starting with the first row. At each stage two copies of the chain run on the tables that agree with the entries fixed so far, each burnt in on its own. The first copy estimates the marginal of every free entry of the row and picks the one whose marginal is closest to zero or one, and the second copy, which never sees that choice, estimates the marginal of the chosen entry at its more likely value. The entry is then fixed at that value, so every factor of the product stays close to one and its log is estimated with small variance. An entry whose column sum is zero or equals the number of remaining rows is forced and costs no stage. When the first row is complete it is removed and its ones are subtracted from the column sums, and the counter continues on the rows below until none remain. A partially fixed first row breaks the irreducibility of plain Curveball trades. Trades that involve the first row therefore pick their partner uniformly among the rows that can exchange something with it, and are accepted with a Metropolis ratio, which keeps the uniform distribution invariant. Each estimated marginal is a time average of its chain, its variance is estimated by batch means, and the reported standard error adds these variances over the stages as if the stages were independent. The total number of trades is set by the budget from a pilot run and divided evenly over the stages, which number between 17 and 591 here and grow linearly with the rows.

Over the 72 runs at 60 minutes the mean error is $-0.002$ with standard error $0.003$, and the error divided by the reported standard error has standard deviation $0.98$. At 10 minutes the mean error is $+0.013$ with standard error $0.005$, a bias that a tenfold longer burn-in does not change and that the sixfold budget removes. MarginFlow's standard deviation is within a factor of $1.3$ of $\sqrt{(e^{D_2}-1)/N}$ from its exact effective sample fraction on every margin.

\subsection{What a draw costs}
\label{app:cost}

Table~\ref{tab:cost} times one repetition of $N=4000$ draws on 13 held-out margins across the size range and turns each time into effective draws per second. The network samples on one A100 and the analytically designed proposals run on the CPU, the standard practice for each. The smallest margin finishes within a tenth of a second for every proposal, and on the other twelve the Harrison--Miller proposal takes 1.2 to 4.7 times longer than MarginFlow. Part of that lead is the hardware, since a third to a half of the Harrison--Miller time goes into its closed-form weights and a GPU implementation would remove most of it. But no implementation changes how many of its draws a proposal keeps, and the last two columns fold that fraction into the time. There MarginFlow delivers 1.2 to 9.2 times the effective draws per second of the Harrison--Miller proposal on the same twelve margins, and the gap is widest, 9.2 on the $40\times20$ power-law margin and 6.2 at $840\times12$, where the Harrison--Miller proposal also loses the most draws. Training is paid once, about 25 hours on a machine with eight A100s per seed (Appendix~\ref{app:training}), and none of it recurs at deployment.

\begin{table}[htbp]
\caption{Seconds for one repetition of $N=4000$ draws on held-out margins across the size range, on one machine with eight CPU threads and one A100 for the network's forward pass, and effective draws per second, the repetition's effective sample size divided by its time. Types per state is the largest number of feasible row types at any state the draws visited. The cap of Section~\ref{sec:setup} is checked on a census of four draws, so a full run can exceed it.}
\label{tab:cost}
\centering
\small
\setlength{\tabcolsep}{3pt}
\begin{tabular}{llccccccc}
\toprule
& & & \multicolumn{4}{c}{Seconds per repetition} & \multicolumn{2}{c}{Effective draws per second} \\
\cmidrule(lr){4-7}\cmidrule(lr){8-9}
Margin & Size & \begin{tabular}[c]{@{}c@{}}Types\\per state\end{tabular} & \begin{tabular}[c]{@{}c@{}}Harrison--\\Miller\end{tabular} & CDHL & Uniform & \begin{tabular}[c]{@{}c@{}}Margin-\\Flow\end{tabular} & \begin{tabular}[c]{@{}c@{}}Harrison--\\Miller\end{tabular} & \begin{tabular}[c]{@{}c@{}}Margin-\\Flow\end{tabular} \\
\midrule
Saturated columns & $12\times10$ & 15 & 0.03 & 0.03 & 0.03 & 0.04 & 143\,000 & 94\,500 \\
Species by site & $20\times15$ & 2050 & 2.6 & 2.6 & 1.6 & 0.57 & 1\,320 & 6\,980 \\
Web of Life & $35\times29$ & 44063 & 13.8 & 13.8 & 11.4 & 9.0 & 282 & 444 \\
Power-law & $40\times20$ & 6924 & 3.8 & 3.7 & 3.6 & 0.90 & 471 & 4\,320 \\
Power-law & $40\times40$ & 14198 & 2.6 & 2.7 & 3.7 & 0.65 & 1\,460 & 6\,160 \\
Near-regular & $40\times40$ & 130722 & 78 & 72 & 112 & 63 & 51.4 & 63.9 \\
Web of Life & $42\times30$ & 20487 & 8.8 & 8.9 & 9.7 & 4.9 & 441 & 808 \\
Species by site & $51\times18$ & 3572 & 6.1 & 5.6 & 4.3 & 1.3 & 618 & 3\,080 \\
Species by site & $53\times28$ & 1713 & 3.8 & 3.5 & 4.8 & 0.94 & 1\,040 & 4\,270 \\
Power-law & $72\times6$ & 20 & 2.4 & 2.3 & 2.7 & 0.71 & 1\,470 & 5\,600 \\
Species by site & $112\times5$ & 10 & 2.6 & 2.1 & 4.1 & 0.82 & 1\,510 & 4\,850 \\
Bimodal & $600\times20$ & 18152 & 367 & 317 & 191 & 104 & 10.9 & 38.5 \\
Power-law & $840\times12$ & 924 & 450 & 368 & 168 & 97 & 6.6 & 41.1 \\
\bottomrule
\end{tabular}
\end{table}

\subsection{Results by family and on the hardest margins}
\label{app:family}

Table~\ref{tab:family} breaks Table~\ref{tab:main} down by evaluation and by family, and Table~\ref{tab:hard} lists the 56 margins of its hardest tier one by one. The gains sit in the families with uneven margins, the extreme sums and the heavier power laws, and in the species-by-site tables. There the tenth percentile of the post-hoc best falls to between 30\% and 58\%, and MarginFlow's stays above 93\%. The near-regular and tight-column families are easy for every proposal, and the two tie on all but one of them. On the hardest tier the configuration that the post-hoc best selects changes from margin to margin, across four of the proposals and three exponents. The three seeds of MarginFlow land within about 3 points of each other on the median margin.

In the configuration column, CP is the conditional Poisson proposal with the ratio weights and CP-B with the weights of \citet{blanchet2009efficient}, HM-C and HM-G are the two proposals of \citet{harrison2013importance}, and ME is maximum entropy, each followed by its exponent $u$.

\begin{table}[htbp]
\caption{Median effective sample fraction (\%) on the 1190 held-out margins by evaluation and by family, with the tenth percentile of the post-hoc best and of MarginFlow. W/T/L is as in Table~\ref{tab:main}.}
\label{tab:family}
\centering
\small
\setlength{\tabcolsep}{3pt}
\begin{tabular}{lrccccccc}
\toprule
 & & & & & & \multicolumn{2}{c}{Tenth percentile} & \\
\cmidrule(lr){7-8}
 & & & Harrison-- & Post-hoc & Margin- & Post-hoc & Margin- & \\
Margins & $n$ & CDHL & Miller & best & Flow & best & Flow & W/T/L \\
\midrule
\multicolumn{9}{l}{\emph{by evaluation}} \\
\quad evaluated exactly & 681 & 65.4 & 98.5 & 99.7 & 99.9 & 88.5 & 98.9 & 176/502/3 \\
\quad estimated from draws & 509 & 12.3 & 92.6 & 98.3 & 99.8 & 35.4 & 95.8 & 239/270/0 \\
\midrule
\multicolumn{9}{l}{\emph{synthetic, new seeds}} \\
\quad power law 0.6 & 106 & 27.1 & 97.1 & 99.7 & 99.9 & 76.7 & 99.5 & 27/79/0 \\
\quad power law 1.0 & 107 & 5.0 & 91.7 & 99.4 & 99.8 & 57.9 & 98.8 & 35/72/0 \\
\quad power law 1.4 & 107 & 4.5 & 82.5 & 99.3 & 99.7 & 38.8 & 97.3 & 39/68/0 \\
\quad near-regular & 111 & 57.5 & 99.4 & 100.0 & 99.9 & 99.2 & 99.8 & 1/110/0 \\
\quad extreme sums & 112 & 73.5 & 86.4 & 88.2 & 99.4 & 30.5 & 94.3 & 85/27/0 \\
\quad bimodal (2) & 109 & 22.1 & 97.2 & 98.9 & 99.8 & 46.5 & 97.5 & 45/64/0 \\
\quad bimodal (2, wide) & 109 & 38.6 & 98.5 & 99.7 & 99.9 & 76.9 & 98.4 & 33/76/0 \\
\quad bimodal (3) & 106 & 24.9 & 97.3 & 98.3 & 99.8 & 81.1 & 99.0 & 51/54/1 \\
\quad tight columns & 13 & 95.1 & 99.8 & 99.8 & 99.9 & 99.5 & 99.7 & 0/13/0 \\
\midrule
\multicolumn{9}{l}{\emph{synthetic, interpolated}} \\
\quad power law 0.8 & 24 & 36.1 & 96.7 & 99.5 & 99.8 & 92.6 & 99.5 & 8/16/0 \\
\quad power law 1.2 & 25 & 37.3 & 93.4 & 99.6 & 99.8 & 82.1 & 99.0 & 9/16/0 \\
\quad bimodal (2) & 26 & 63.7 & 99.1 & 99.8 & 99.9 & 66.9 & 84.9 & 7/19/0 \\
\quad bimodal (3) & 25 & 67.3 & 99.1 & 99.8 & 99.9 & 97.1 & 99.7 & 4/21/0 \\
\midrule
\multicolumn{9}{l}{\emph{real tables}} \\
\quad species by site & 88 & 42.6 & 94.8 & 97.4 & 99.6 & 53.4 & 94.8 & 43/43/2 \\
\quad Web of Life & 117 & 62.0 & 98.0 & 99.6 & 99.8 & 89.6 & 98.9 & 26/91/0 \\
\quad psychometric and social & 5 & 4.0 & 92.5 & 98.6 & 99.4 & 48.8 & 90.7 & 2/3/0 \\
\bottomrule
\end{tabular}
\end{table}

\begin{table}[p]
\caption{The 56 held-out margins on which the post-hoc best of the 31 configurations keeps under 37\% of its draws, ordered by that fraction. Effective sample fractions in \%, MarginFlow as the median of its three seeds followed by the three values.}
\label{tab:hard}
\centering
\fontsize{8.5}{10.2}\selectfont
\setlength{\tabcolsep}{2.3pt}
\begin{tabular}{llccccccc}
\toprule
 & & & & Harrison-- & Post-hoc & & Margin- & \\
Margins & Size & Density & CDHL & Miller & best & Configuration & Flow & Seeds \\
\midrule
power law 1.4 & $840\times6$ & dense & 0.1 & 0.6 & 0.4 & HM-G 1 & 94.0 & 96.5, 93.2, 94.0 \\
extreme sums & $840\times12$ & dense & 0.2 & 0.7 & 0.5 & HM-C 1 & 57.6 & 61.1, 32.7, 57.6 \\
power law 1.4 & $420\times6$ & dense & 0.1 & 2.7 & 1.4 & ME 0.75 & 98.2 & 97.6, 98.2, 98.2 \\
power law 1.4 & $560\times8$ & dense & 0.2 & 0.7 & 1.6 & HM-C 1 & 94.0 & 97.7, 93.2, 94.0 \\
bimodal (2) & $600\times20$ & dense & 0.0 & 0.4 & 2.3 & ME 1 & 48.8 & 74.1, 48.8, 45.9 \\
bimodal (2) & $840\times12$ & dense & 0.2 & 2.3 & 2.5 & HM-C 1 & 94.1 & 95.4, 94.1, 91.9 \\
bimodal (2, wide) & $840\times12$ & dense & 0.1 & 2.1 & 2.8 & HM-C 1 & 84.6 & 88.6, 83.9, 84.6 \\
extreme sums & $840\times6$ & dense & 0.3 & 3.6 & 3.0 & HM-C 1 & 95.6 & 95.6, 89.9, 96.0 \\
power law 1.4 & $840\times12$ & dense & 0.2 & 1.6 & 3.3 & HM-C 1 & 92.3 & 93.5, 90.8, 92.3 \\
species by site & $152\times15$ & -- & 0.2 & 1.2 & 3.8 & HM-G 1 & 77.6 & 83.1, 77.6, 73.7 \\
power law 1.0 & $840\times12$ & dense & 0.2 & 5.3 & 4.3 & HM-C 1 & 95.5 & 96.7, 95.1, 95.5 \\
power law 1.0 & $560\times8$ & dense & 0.4 & 5.9 & 4.8 & HM-C 1 & 98.0 & 98.4, 98.0, 97.7 \\
power law 1.0 & $420\times6$ & dense & 0.3 & 6.2 & 5.0 & ME 0.75 & 98.8 & 98.4, 98.8, 98.9 \\
bimodal (3) & $840\times12$ & dense & 0.3 & 4.7 & 5.3 & HM-C 1 & 96.8 & 98.2, 95.7, 96.8 \\
bimodal (2) & $144\times12$ & dense & 0.3 & 4.4 & 6.1 & ME 1 & 59.6 & 74.1, 56.6, 59.6 \\
extreme sums & $600\times20$ & mid & 1.0 & 6.8 & 6.1 & HM-C 1 & 89.2 & 89.2, 85.9, 90.7 \\
power law 1.0 & $840\times6$ & dense & 0.1 & 4.6 & 6.2 & HM-C 1 & 98.5 & 98.5, 98.6, 98.4 \\
extreme sums & $700\times5$ & dense & 0.5 & 6.7 & 6.3 & HM-C 1 & 95.0 & 96.5, 89.9, 95.0 \\
species by site & $198\times15$ & -- & 0.2 & 2.3 & 6.8 & HM-G 1 & 83.1 & 83.1, 83.9, 79.7 \\
bimodal (2) & $360\times12$ & dense & 0.4 & 7.3 & 7.5 & HM-C 1 & 87.4 & 94.2, 87.4, 85.9 \\
extreme sums & $560\times8$ & dense & 0.8 & 7.5 & 8.2 & HM-C 1 & 85.6 & 85.3, 86.5, 85.6 \\
extreme sums & $840\times12$ & mid & 1.1 & 7.9 & 8.7 & HM-C 1 & 89.5 & 92.4, 89.5, 89.2 \\
power law 1.4 & $700\times5$ & dense & 0.1 & 3.0 & 9.5 & HM-G 1 & 98.2 & 97.2, 98.6, 98.2 \\
extreme sums & $600\times20$ & dense & 1.0 & 9.9 & 9.8 & HM-C 1 & 82.1 & 89.8, 82.1, 78.1 \\
bimodal (2, wide) & $240\times20$ & dense & 0.1 & 4.1 & 9.9 & ME 1 & 59.1 & 68.9, 59.1, 57.5 \\
species by site & $82\times19$ & -- & 2.0 & 10.0 & 10.0 & HM-C 1 & 89.7 & 92.9, 89.7, 88.1 \\
bimodal (2, wide) & $360\times12$ & dense & 0.2 & 8.7 & 10.1 & HM-C 1 & 88.1 & 92.1, 88.1, 87.9 \\
power law 1.0 & $700\times5$ & dense & 0.2 & 3.6 & 10.2 & HM-G 1 & 98.8 & 98.4, 99.0, 98.8 \\
power law 1.4 & $180\times6$ & dense & 0.5 & 3.0 & 10.4 & HM-G 1 & 98.3 & 98.3, 98.5, 98.1 \\
extreme sums & $360\times12$ & dense & 1.5 & 11.4 & 12.6 & HM-C 1 & 96.5 & 97.1, 94.8, 96.5 \\
extreme sums & $420\times6$ & dense & 0.9 & 12.9 & 12.7 & HM-C 1 & 96.9 & 97.4, 96.6, 96.9 \\
power law 1.4 & $360\times12$ & dense & 1.2 & 14.4 & 13.1 & HM-C 1 & 96.1 & 96.7, 93.6, 96.1 \\
bimodal (2) & $240\times20$ & dense & 0.1 & 3.6 & 13.4 & ME 1 & 69.6 & 87.7, 69.6, 61.7 \\
extreme sums & $560\times4$ & dense & 0.1 & 17.6 & 17.6 & HM-C 1 & 98.1 & 99.0, 97.6, 98.1 \\
power law 1.4 & $240\times8$ & dense & 1.3 & 17.7 & 18.2 & HM-C 1 & 98.7 & 98.7, 98.1, 98.8 \\
bimodal (2) & $560\times8$ & dense & 1.3 & 20.9 & 20.2 & HM-C 1 & 97.6 & 97.8, 97.6, 96.1 \\
bimodal (2, wide) & $100\times20$ & dense & 0.3 & 9.1 & 22.7 & ME 1 & 71.3 & 84.5, 70.1, 71.3 \\
bimodal (2) & $60\times12$ & dense & 1.6 & 19.8 & 24.3 & HM-C 1.25 & 79.5 & 88.3, 79.5, 77.9 \\
species by site & $37\times26$ & -- & 1.9 & 13.6 & 25.0 & CP-B 1.25 & 68.4 & 80.3, 68.4, 64.1 \\
bimodal (2) & $700\times5$ & dense & 0.9 & 26.4 & 27.1 & HM-C 1 & 93.9 & 98.4, 93.9, 93.3 \\
extreme sums & $360\times12$ & mid & 7.4 & 27.8 & 27.4 & HM-C 1 & 95.3 & 96.4, 93.6, 95.3 \\
power law 1.4 & $350\times5$ & dense & 0.3 & 5.5 & 30.1 & HM-G 1 & 98.9 & 98.3, 98.9, 99.2 \\
extreme sums & $840\times6$ & sparse & 17.2 & 29.6 & 30.3 & HM-C 1 & 93.6 & 95.8, 93.5, 93.6 \\
bimodal (3) & $100\times20$ & dense & 3.1 & 31.2 & 30.7 & HM-C 1 & 92.4 & 93.4, 92.4, 90.5 \\
bimodal (2) & $840\times6$ & dense & 0.8 & 31.5 & 30.7 & HM-C 1 & 99.1 & 99.1, 98.8, 99.2 \\
power law 1.4 & $60\times12$ & dense & 10.4 & 34.5 & 30.8 & HM-C 1 & 95.7 & 96.6, 92.3, 95.7 \\
extreme sums & $240\times8$ & dense & 10.8 & 31.2 & 31.9 & HM-C 1 & 80.1 & 82.9, 79.3, 80.1 \\
bimodal (2) & $144\times12$ & dense & 4.0 & 32.6 & 32.5 & HM-C 1 & 93.3 & 97.4, 93.3, 92.1 \\
bimodal (2, wide) & $144\times12$ & dense & 1.2 & 33.1 & 34.1 & HM-C 1 & 93.7 & 95.9, 93.7, 93.1 \\
bimodal (3) & $360\times12$ & dense & 1.9 & 34.3 & 34.1 & HM-C 1 & 97.9 & 98.7, 97.7, 97.9 \\
bimodal (2) & $240\times8$ & dense & 4.6 & 35.1 & 34.3 & HM-C 1 & 97.5 & 97.7, 97.5, 95.0 \\
extreme sums & $240\times20$ & mid & 11.7 & 36.1 & 35.2 & HM-C 1 & 93.6 & 93.6, 87.4, 94.4 \\
Web of Life & $164\times18$ & -- & 1.6 & 36.7 & 35.4 & HM-C 1 & 94.9 & 95.0, 92.9, 94.9 \\
power law 1.0 & $240\times8$ & dense & 1.7 & 33.5 & 36.2 & HM-C 1 & 99.1 & 99.1, 98.8, 99.2 \\
species by site & $650\times4$ & -- & 0.1 & 1.3 & 36.6 & HM-G 1 & 96.4 & 91.8, 96.4, 97.5 \\
extreme sums & $560\times8$ & mid & 14.8 & 36.6 & 36.8 & HM-C 1 & 98.7 & 98.7, 96.5, 98.9 \\
\bottomrule
\end{tabular}
\end{table}

\FloatBarrier
\section{Discussion}
\label{app:discussion}

In use MarginFlow stands where the analytically designed proposal stands. An ecologist testing nestedness draws matrices with the observed margins and averages a statistic over them, a psychometrician doing conditional inference in the Rasch model draws item-response tables with the observed scores, and a social scientist tests a motif count against affiliation networks with the same degrees. Each of these is a weighted average over draws. The network makes the draws and returns the weights, and the user needs fewer of them for the same error. The same draws and weights feed sequential Monte Carlo with resampling, and the proposal serves as the independence proposal of a Metropolis--Hastings chain that leaves the uniform distribution invariant.

The zero-shot network is also a starting point. Section~\ref{sec:count} shows that the loss trains on one margin from its own draws, without the count, so a user with one hard margin can continue training on it. The loss reports the variance of the log weights on that margin as it falls, and the user watches the proposal improve without ever knowing the answer.

The reach of the network is set by how it reads a state. It sees the feasible row types, and the softmax over them makes every step exact and the weighted count unbiased. The number of types, however, grows with the number of distinct reduced column sums. We train where it stays under $2\times10^4$ per state and evaluate where it stays under $10^5$, on a test set drawn under the same cap for every proposal in the comparison. Where the types are far more numerous, the policy of Appendix~\ref{app:training} that places the ones of a row group by group is exact without any enumeration of types, at a higher cost per training step, and it is the natural next version of the network.

MarginFlow itself reaches further. It rests on two properties, that a matrix is built one row at a time and that the remainder after each row is again an instance of the problem, and both hold beyond 0-1 matrices with fixed margins. Contingency tables with integer entries, graphs with prescribed degrees, and tables with structural zeros share them. The network then has to read a column through more than its reduced sum, since a structural zero, or the symmetry of a graph's adjacency matrix, fixes some entries of a column and leaves others free. With that input, one network again serves every instance of each. Further out, perfect matchings in bipartite graphs, Latin rectangles, and constrained lattice paths are built the same way, with a feasibility check at every piece. Each of them has its proposals designed by hand, one family at a time, and MarginFlow can learn one proposal for the whole family from the draws the sampler makes anyway.

\section{Extended related work}
\label{app:related}

\paragraph{Sampling and counting with fixed margins.}
Exact dynamic programming over the multiset of reduced column sums gives the exact count and exact uniform draws \citep{miller2013exact}, and its state graph is the type graph of Section~\ref{sec:network}. Markov-chain methods target the uniform distribution over matrices with fixed margins \citep{verhelst2008efficient,gotelli2012statistical,strona2014fast,fosdick2018configuring}. Polynomial mixing bounds for the swap chain cover restricted classes of margins \citep{kannan1999simple,erdos2022mixing}, and \citet{fu2026spectral} establish such a bound for the lazy swap chain under arbitrary feasible margins. \citet{nie2026snake} introduce the Snake sampler, prove an analogous guarantee for its lazy version, and compare it empirically with SIS in fixed-margin testing. Every SIS proposal since \citet{snijders1991enumeration} is analytically designed, from the conditional Poisson of \citet{chen2005sequential}, through the sparse-regime proposal of \citet{blanchet2009efficient}, the dead-end-free construction of \citet{blitzstein2011sequential} for graphs, and the asymptotic-enumeration family of \citet{harrison2013importance}, to the maximum entropy proposal of \citet{glasserman2023maximum}. \citet{bezakova2012negative} show margins on which the conditional Poisson proposal needs exponentially many draws to estimate the count, and the hard tier of Section~\ref{sec:main} holds margins on which every configuration of the sweep fails. For contingency tables with integer entries, \citet{jerdee2024improved} sharpen the count approximation behind the proposal, and the proposal stays analytically designed.

\paragraph{Learned proposals and amortized samplers.}
Learning a proposal from the sampler's own draws predates GFlowNets. \citet{gu2015neural} train the proposal of sequential Monte Carlo on its weighted particles, \citet{muller2019neural} train normalizing flows as importance samplers, \citet{wu2019solving} train autoregressive networks on lattice models, and \citet{nicoli2020asymptotically} correct such networks by importance weights so the weighted draws estimate the partition function without bias. Each trains one network per model. GFlowNets \citep{bengio2021flow,bengio2023gflownet} with trajectory balance \citep{whitammer2022trajectory} do the same for one reward, and on policy their expected gradient is that of the reverse Kullback-Leibler divergence \citep{richter2020vargrad,whitammer2023gflownets}, which Theorem~\ref{thm:grad} uses. The same objective now fine-tunes language models, where \citet{chen2026powerflow} match the policy to a power of the base model with a length-aware trajectory balance. \citet{deleu2024control} give the correction that keeps the terminal distribution right when an object has many construction paths, a case the row-by-row tree here excludes. \citet{zhao2024twisted} learn the twist of sequential Monte Carlo, the expected future weight of a partial sequence, which is the role the completion count plays here. \citet{choi2026reinforced} use the learned policy as the proposal kernel of sequential Monte Carlo, with a twist from the learned value function, one policy per target. Here the reward is one on every matrix, so the exact policy is a ratio of completion counts and the plain importance weights carry the count without resampling. MarginFlow also shares one policy across margins. \citet{boussif2025action} shorten long GFlowNet trajectories by merging recurring action sequences into one action, whereas here the action is a whole row from the outset and Proposition~\ref{prop:reduced} merges rows into types by the symmetry that the exact recursion already uses.

\paragraph{Amortization across instances.}
Inference compilation trains one proposal across the runs of a probabilistic program \citep{paige2016inference,le2017inference}, with simulations as the instances. \citet{zhang2023robust} condition a GFlowNet on a scheduling instance and train it across instances with the log-variance objective, which Section~\ref{sec:prelim} adopts, and \citet{kim2025gfacs} train one per problem class across its instances as a prior for ant colony search on seven combinatorial optimization problems. Every partial matrix met while sampling one margin is again an instance with its own reduced margins, so a network reading the remaining margins learns from every state of every trajectory through the pool. Trained this way on 1904 margins, MarginFlow runs zero-shot on 1190 held-out margins. Appendix~\ref{app:discussion} describes the extensions where a column carries more than its reduced sum.

\end{document}